%% file: arxiv.tex
\documentclass[11pt]{article}

\usepackage{times}
\usepackage{graphicx}
\usepackage{natbib}
\usepackage{titletoc}
\usepackage[toc,page,header]{appendix}
\usepackage[T1]{fontenc}
\usepackage{microtype}
\usepackage{amsmath,amssymb,amsthm,mathtools}
\usepackage{aliascnt}
\usepackage{enumitem}
\usepackage{xspace}
\usepackage{algorithm}
\usepackage{float}
\usepackage{flafter}
\usepackage{placeins}
\usepackage{needspace}
\usepackage[noend]{algpseudocode}
\usepackage{xcolor}
\usepackage[colorlinks=true,linkcolor=blue,citecolor=blue,urlcolor=black,linktocpage=true]{hyperref}
\usepackage[nameinlink,capitalise]{cleveref}

\algnewcommand{\algorithmicdefine}{\textbf{Define:}}
\algnewcommand{\Define}{\item[\algorithmicdefine]}
\algnewcommand{\algorithmicparameters}{\textbf{Parameters:}}
\algnewcommand{\Parameters}{\item[\algorithmicparameters]}

\makeatletter
\providecommand{\theHALG@line}{}
\renewcommand{\theHALG@line}{\thealgorithm.\arabic{ALG@line}}
\makeatother

\newcommand{\compactvspace}[1]{}

\allowdisplaybreaks

\newtheorem{theorem}{Theorem}[section]

\newaliascnt{lemma}{theorem}
\newtheorem{lemma}[lemma]{Lemma}
\aliascntresetthe{lemma}
\crefname{lemma}{Lemma}{Lemmas}
\Crefname{lemma}{Lemma}{Lemmas}

\newaliascnt{corollary}{theorem}

\aliascntresetthe{corollary}
\crefname{corollary}{Corollary}{Corollaries}
\Crefname{corollary}{Corollary}{Corollaries}

\newaliascnt{proposition}{theorem}

\aliascntresetthe{proposition}
\crefname{proposition}{Proposition}{Propositions}
\Crefname{proposition}{Proposition}{Propositions}

\newaliascnt{assumption}{theorem}

\aliascntresetthe{assumption}
\crefname{assumption}{Assumption}{Assumptions}
\Crefname{assumption}{Assumption}{Assumptions}

\newcommand{\R}{\mathbb{R}}
\newcommand{\E}{\mathbb{E}}
\newcommand{\Pp}{\mathbb{P}}
\newcommand{\cS}{\mathcal{S}}
\newcommand{\cA}{\mathcal{A}}
\newcommand{\cD}{\mathcal{D}}
\newcommand{\cK}{\mathcal{K}}
\newcommand{\cKhat}{\widehat{\mathcal{K}}}
\newcommand{\cO}{\mathcal{O}}
\newcommand{\cF}{\mathcal{F}}

\newcommand{\Ball}{\mathbb{B}}
\newcommand{\bI}{\mathbf{I}}
\newcommand{\xbar}{\overline{x}}
\newcommand{\xhat}{\widehat{x}}
\newcommand{\phihat}{\widehat{\phi}}
\newcommand{\ip}[2]{\left\langle #1,#2\right\rangle}
\newcommand{\inner}[1]{\langle #1\rangle}

\newcommand{\norm}[1]{\left\lVert #1\right\rVert}

\newcommand{\ind}{\mathbf{1}}
\newcommand{\vol}{\operatorname{Vol}}

\newcommand{\poly}{\operatorname{poly}}
\newcommand{\Reg}{\operatorname{Reg}}
\newcommand{\op}{\mathrm{op}}
\newcommand{\feat}{\mathrm{ft}}
\newcommand{\pl}{\mathrm{pl}}

\newcommand{\Alg}[1]{\textnormal{\scshape #1}}
\newcommand{\moment}{\Alg{MomentEstimate}\xspace}
\newcommand{\certified}{\Alg{ApproxLO}\xspace}
\newcommand{\compile}{\Alg{ApproxSEP}\xspace}

\DeclareMathOperator*{\argmin}{arg\,min}
\DeclareMathOperator{\CLIP}{CLIP}
\DeclareMathOperator{\Proj}{Proj}

\title{A Computationally Efficient and Rate-Optimal Algorithm 
for Adversarial Linear MDPs 
with (Aggregate) Bandit Feedback}
\author{%
    Haolin Liu\thanks{Email: \href{mailto:srs8rh@virginia.edu}{\texttt{srs8rh@virginia.edu}}.} \\
    \scalebox{0.9}{University of Virginia}
    \and
    Chen-Yu Wei\thanks{Email: \href{mailto:chenyu.wei@virginia.edu}{\texttt{chenyu.wei@virginia.edu}}.} \\
    \scalebox{0.9}{University of Virginia}
    \and
    Hanbin Zhou\thanks{Email: \href{mailto:hanbin.zhou@connect.hkust-gz.edu.cn}{\texttt{hanbin.zhou@connect.hkust-gz.edu.cn}}.} \\
    \scalebox{0.9}{Hong Kong University of}\\\scalebox{0.9}{Science and Technology}
}
\date{}

\begin{document}
\addtocontents{toc}{\protect\setcounter{tocdepth}{-1}}
\maketitle

\begin{abstract}
We study episodic linear Markov decision processes with unknown transitions, adversarial losses, and bandit feedback. The state-of-the-art result for this setting \citep{liu2024linear} either achieves $\sqrt{K}$ regret with a computationally inefficient algorithm or achieves suboptimal $K^{3/4}$ regret, where $K$ is the number of episodes. Our work gives the first $\sqrt{K}$ regret with a polynomial-time algorithm. 

Our method reduces adversarial linear MDPs to adversarial linear bandits whose action set $\cK$ consists of policies' expected feature vectors. 
The central challenge is that $\cK$ is unknown and can be estimated only through exploration. We address this challenge by constructing an \emph{approximate separation oracle} for~$\cK$ that either generates executable policies from queried features, or refines an outer approximation of~$\cK$. This approach draws inspiration from the reduction of \citet{vanerven2025reduction} for linear contextual bandits and the oracle-efficient linear bandit algorithm of \citet{ito2019oracle}, but additionally deals with the imprecision of the oracle. The resulting algorithm requires only aggregate loss feedback, and to our knowledge. This is the first polynomial-time algorithm for adversarial MDPs with aggregate feedback beyond the tabular case. 
\end{abstract}

\input{sections/main_01_intro}

\input{sections/main_02_model}
\input{sections/main_03_algorithm}
\input{sections/main_04_compiler}
\input{sections/main_06_computation}

\label{page:main-text-end}

\clearpage

\bibliography{references}
\bibliographystyle{plainnat}
\clearpage

\appendix
\appendixpage
\addtocontents{toc}{\protect\setcounter{tocdepth}{2}}
\input{appendices/contents}
\input{appendices/related_work}
\FloatBarrier
\input{appendices/03_master_algorithm}
\FloatBarrier
\input{appendices/E_outer_details}
\FloatBarrier
\input{appendices/05_support}
\FloatBarrier
\input{appendices/07_computation}

\FloatBarrier
\input{appendices/06_regret}

\end{document}

%% file: sections/main_01_intro.tex
\section{Introduction}
\label{sec:paper-introduction}

Achieving rate-optimal regret with polynomial computation is a central challenge in adversarial linear Markov decision processes (MDPs) with bandit feedback and unknown transitions. Starting with the $K^{14/15}$ regret bound of \citet{luo2021policy}, subsequent works improved the rate to $K^{8/9}$, $K^{6/7}$, and $K^{3/4}$ \citep{dai2023refined,sherman2023improved,liu2024linear}, respectively. Meanwhile, \citet{liu2024linear} obtained the optimal $\sqrt K$ rate using a computationally inefficient method. Whether $\sqrt K$ regret can be achieved with an efficient (polynomial-time) algorithm has remained open.

Our work closes this gap by giving the first computationally efficient algorithm with $\poly(d,H)\sqrt{K}$ regret, where $d$ is the feature dimension and $H$ is the horizon length. In fact, our work handles the more challenging \emph{aggregate feedback} model: the learner, despite observing the state sequence on the trajectory, only observes the total episode loss rather than the individual step losses.  Prior work for the aggregate feedback model focuses either on the tabular case \citep{cohen2021, lancewicki2025, ito2025} or the stochastic-loss setting \citep{cassel2024near, chen2026outcome}. To our knowledge, our work gives the first polynomial-time algorithm for aggregate feedback model with adversarial losses beyond the tabular case.

Our approach builds on a reduction from linear MDPs to linear bandits. For each policy $\pi$, we stack its expected state-action features across layers as a \emph{policy-feature vector} $\phi^\pi$. Because the expected episode loss is linear in $\phi^\pi$, the MDP can be viewed as a linear bandit whose action set consists of policy-feature vectors. This action set is unknown because $\phi^\pi$ depends on the unknown transition kernel.  Existing methods based on this reduction \citep{kong2024improved,liu2024linear} discretize the policy class, estimate the policy-feature of every policy, and apply a linear-bandit algorithm to the resulting finite action set. While this achieves the optimal $\sqrt K$ regret rate \citep{liu2024linear}, their computation scales with an exponentially large policy class.

Under the same reduction idea, \citet{vanerven2025reduction} develop an efficient reduction from linear contextual bandits to linear bandits. They construct a convex approximation $\cKhat$ of $\cK$ from past i.i.d. contexts and construct a linear optimization oracle for $\cKhat$. This oracle allows them to run a linear bandit algorithm (continuous exponential weights) over $\cKhat$ without explicitly representing the set. Their linear contextual bandit setting is a special case of linear MDP when $H=1$. 

Although the approach of \citet{vanerven2025reduction} suggests a promising route for linear MDPs, a direct extension faces a major obstacle: the straightforward estimate $\cKhat$ for linear MDP may not be convex due to the recursive structure of MDPs. Therefore, the log-concave sampling underlying continuous exponential weights no longer applies, and we also lose an exact optimization oracle over $\cKhat$. To address this, we leverage the idea of \citet{ito2019oracle}: we run continuous exponential weights over an outer convex set $\cO\supseteq\cK$ represented as an intersection of halfspaces (cuts).   In \cite{ito2019oracle}, whenever a point queried from $\cO$ lies outside $\cK$, an exact oracle generates a separating cut to refine $\cO$. 
Unlike in their setting, however, we lack an exact oracle for~$\cK$. Thus, we construct cuts that exclude the queried point only when doing so preserves  $\cO\supseteq\cK$ with high probability. We remark that while \citet{ito2019oracle} propose this mechanism to reduce oracle complexity, we use it for a completely different purpose of handling unknown action sets.%

When the algorithm does not return a cut, it decomposes the queried point into a distribution over policies whose average policy feature approximates the query point. This will be the policy distribution played by the learner in that episode. Together, the two cases form a \emph{approximate separation procedure}: it either constructs a finite policy mixture approximating the query, or returns a separating cut that refines $\cO$. We implement it using a Frank--Wolfe-style procedure. %
Besides a distribution over policies, our procedure simultaneously generates uncertainty measures, which will be crucial in constructing a bias-correction bonus fed to the learning procedure. %

The paper is organized as follows. In \cref{sec:paper-model} we describe the problem setting. In \cref{sec: algorithm overview}, we give a high-level overview of the main algorithm. In \cref{sec:paper-algorithm}, we provide a sketch of regret analysis. In \cref{sec:paper-compiler}, we focus on the design of the approximate separation procedure. In \cref{sec:paper-computation}, we analyze the computational complexity of the algorithm.

%% file: sections/main_02_model.tex
\section{Preliminaries}
\label{sec:paper-model}
We consider a finite-horizon Markov decision process with horizon $H$, layer-wise state sets $\cS_1,\ldots,\cS_{H+1}$, a finite nonempty action set $\cA$ with $A=|\cA|$, and a fixed initial state $s_1\in\cS_1$. For each layer $h$, the transition kernel $P_h(\cdot\mid s,a)$ is a probability distribution on $\cS_{h+1}$. The kernels are unknown to the learner and fixed across episodes. A Markov policy $\pi=(\pi_h)_{h=1}^H$ consists of maps $\pi_h:\cS_h\to\Delta(\cA)$, where $\Delta(\cA)$ denotes the probability simplex over actions.

The learner interacts with the MDP for $K$ episodes. Before interaction, an oblivious adversary fixes loss functions $\ell_{k,h}:\cS_h\times\cA\to[0,1]$ for all episodes and layers. At the start of episode $k$, the learner selects a possibly random Markov policy $\pi_k$ using past observations. Starting from $S_{k,1}=s_1$, the trajectory is generated by
\( A_{k,h}\sim\pi_{k,h}(\cdot\mid S_{k,h}),
 S_{k,h+1}\sim P_h(\cdot\mid S_{k,h},A_{k,h}),
\) for every $h$. We assume the loss feedback observed by the learner consists only of the total episode loss
\(
 L_k=\sum_{h=1}^H\ell_{k,h}(S_{k,h},A_{k,h}),
\)
revealed at the end of the episode. We refer to this observation model as \emph{aggregate bandit feedback}. 

For a Markov policy $\pi$, define its expected loss in episode $k$ by
\(
 V_k^\pi=\E_{P,\pi}\big[\sum_{h=1}^H\ell_{k,h}(S_h,A_h)\big].
\)
Here $\E_{P,\pi}$ denotes expectation over the trajectory generated by following policy $\pi$ under transition kernels $P$.
The objective we consider is to minimize regret relative to the best fixed Markov policy,
\[
 \Reg_K=\E\!\left[\sum_{k=1}^K V_k^{\pi_k}\right]
       -\min_{\pi\text{ Markov}}\sum_{k=1}^K V_k^\pi.
\]
\paragraph{Linear MDPs.}
We assume that transitions and losses are linear in known feature maps $\phi_h:\cS_h\times\cA\to\R^d$. Specifically, there exist unknown vectors $\psi_h(s')\in\R^d$ and $\theta_{k,h}\in\R^d$ such that, for every $k,h,s,a,s'$ with $s\in\cS_h$ and $s'\in\cS_{h+1}$,
\[
 P_h(s'\mid s,a)=\phi_h(s,a)^\top\psi_h(s'),
 \qquad
 \ell_{k,h}(s,a)=\phi_h(s,a)^\top\theta_{k,h}.
\]
The transition coefficients $\psi_h$ are fixed, while the loss coefficients $\theta_{k,h}$ may vary across episodes. We assume $\norm{\phi_h(s,a)}_2\le1$, $\norm{\theta_{k,h}}_2\le\sqrt d$, and $\big\|\sum_{s'\in\cS_{h+1}} \psi_h(s')f(s')\big\|_2\leq \sqrt{d}\|f\|_\infty$ for any $f: \mathcal{S}_{h+1}\to \mathbb{R}$ without loss of generality \citep{jin2020, wei2021learning}. For any $f:\cS_{h+1}\to\R$, define
\(
 w_h(f)=\sum_{s'\in\cS_{h+1}}\psi_h(s')f(s')
\)
and 
\(
 P_hf(s,a)=\E_{s'\sim P_h(\cdot\mid s,a)}[f(s')]
          =\phi_h(s,a)^\top w_h(f).
\)

Let $p=dH$ be the expected-feature dimension, obtained by stacking the $d$-dimensional features across all $H$ layers. Set $\theta_k=(\theta_{k,1},\ldots,\theta_{k,H})$, and let $\phi^\pi=(\phi_1^\pi,\ldots,\phi_H^\pi)$ with $\phi_h^\pi=\E_{P,\pi}[\phi_h(S_h,A_h)]$. Then $\norm{\phi^\pi}_2\le\sqrt H$, $\norm{\theta_k}_2\le\sqrt{dH}$, and
\(
 \E[L_k\mid\pi_k=\pi]=\langle\phi^\pi,\theta_k\rangle=V_k^\pi.
\)
Let $\cK=\{\phi^\pi:\pi\text{ is a Markov policy}\}$ be the set of feasible policy-feature vectors. The policy-feature reduction of \citet{liu2024linear} relates regret to linear losses over this set through
\begin{align}
\Reg_K = \E\left[\sum_{k=1}^K \left\langle \phi^{\pi_k}, \theta_k \right \rangle -\min_{u \in \mathcal{K}}\sum_{k=1}^K \left\langle u, \theta_k \right \rangle \right]
\label{eq:reg-red}
\end{align}
The set $\cK$ depends on the unknown transitions. The reduction therefore yields an adversarial linear bandit problem over an unknown feasible set. The algorithm developed in \cref{sec:paper-algorithm,sec:paper-compiler} achieves the following guarantee.

\begin{theorem}[Main result]
\label{thm:paper-main}
For adversarial linear MDPs with aggregate bandit feedback, Algorithm~\ref{alg:paper-core}, with the complete implementation in Algorithm~\ref{alg:master another}, achieves $\Reg_K= \widetilde O(d^5H^{9/2}\sqrt K)$. Total computation is polynomial in $K,d,H,A$.
\end{theorem}

\paragraph{Notation.}
For a positive integer $m$, write $[m]=\{1,\ldots,m\}$, let $\bI_m$ be the $m\times m$ identity matrix, and denote the Euclidean unit ball by $\Ball_2^m=\{x\in\R^m:\norm{x}_2\le1\}$. The indicator $\ind\{E\}$ equals one when $E$ holds and zero otherwise. For $\Sigma\succeq0$, $\norm{x}_{\Sigma}\triangleq\sqrt{x^\top\Sigma x}$. For a nonempty closed convex set $C\subseteq\R^m$, write $\Proj_C(x)=\argmin_{z\in C}\norm{z-x}_2^2$. For scalars $a\le b$, define the clipping operator by
\(
 \CLIP_{[a,b]}(t)=\min\{b,\max\{a,t\}\}.
\)
For vectors, $\CLIP_{[a,b]^m}$ applies this operator to each coordinate. For functions, clipping is applied pointwise. Throughout, $\widetilde O$ hides logarithmic factors in $K,d,H,A$.

%% file: sections/main_03_algorithm.tex
\section{Algorithm Overview}\label{sec: algorithm overview}
The main algorithm is \cref{alg:paper-core}. The algorithm maintains an outer approximation $\cO\supset \cK$ and a cumulative loss estimate $Z_{k-1}\in\R^p$. Since $\cK$ may have empty interior, we actually require $\cO$ to contain the enlarged set
$
 \cK_\zeta=\cK+\zeta\Ball_2^p,
$
for $\zeta=1/\sqrt{K}$ to ensure that the linear bandit algorithm is run over a set with non-zero volume. Each episode consists of the following steps:  

\begin{algorithm}[t]
\caption{Continuous exponential weights with policy compilation and adaptive cuts}
\label{alg:paper-core}
\begin{algorithmic}[1]
\Require $c_0=\frac{1}{16e}$, $\zeta=K^{-1/2}$,
$R=\sqrt H+\zeta$, $M=K$, and
$\eta\in (0, \frac{c_0}{128\sqrt{d^3H^4}}]$. 
\State $Z_0\gets0$; $\cO\gets R\Ball_2^p$; $\cD_{1,h}\gets\varnothing$ for all $h\in[H]$
\For{$k=1,\ldots,K$}
  \Repeat
    \State \textit{restart}$\gets\Alg{false}$
    \State Define $q_k(x)\propto\exp(-\eta Z_{k-1}^\top x)\ind\{x\in\cO\}$ \label{line:master-exponential-weights}
    \State \mbox{Estimate mean and covariance: $(\widehat\mu_k,\widehat\Sigma_k)\gets\Alg{MomentEstimate}(\cO,Z_{k-1})$ \label{line: moment estimate}(\cref{app:moment-details})}
    \State Let $\{Y_{k,i}\}_{i=1}^{2p}$ be the set $\left\{ \widehat\mu_k + c_0 \widehat\Sigma_k^{1/2} v:~ v \in \{\pm e_1, \pm e_2, \ldots, \pm e_p\} \right\}$. \label{line:master-targets}
    \For{$i=1,\ldots,2p$}
      \State $\mathsf{reply}\gets\compile(Y_{k,i},(\cD_{k,h})_{h=1}^H,M,\zeta)$ (\cref{alg:paper-compiler}) \label{line:reply compile}
      \If{$\mathsf{reply}$ is a cut $(c,b)$}
        \State $\cO\gets\cO\cap\{x:c^\top x\ge b\}$
        \State \textit{restart}$\gets\Alg{true}$; \textbf{break}  \label{line: restart line}
      \Else
        \State $((\pi_{k,i,j})_{j=1}^M,\varepsilon_{k,i})\gets\mathsf{reply}$ 
      \EndIf
    \EndFor
  \Until{\textit{restart}$=\Alg{false}$}
  \State Draw $(I_k,J_k)\sim\operatorname{Unif}([2p]\times[M])$ and execute $\pi_k = \pi_{k,I_k, J_k}$.  \label{line:master-draw}
  \State Observe feature vectors along the trajectory $(\phi_{k,h})_{h=1}^H$ and the total loss $L_k$.  \label{line:master-rollout}
  \State $\widehat\theta_k\gets \left(\zeta^2 \mathbf{I}_p + \frac{c_0^2}{p}\widehat\Sigma_k\right)^{-1}(Y_{k,I_k}-\widehat\mu_k)L_k$ \label{line:master-estimate}
  \State $Z_k\gets Z_{k-1}+\widehat\theta_k-B_k$ where $B_k=\frac{10\eta p \varepsilon_k}{c_0}Z_{k-1}$ with $\varepsilon_k=\sqrt{\frac1{2p}\sum_{i=1}^{2p}\varepsilon_{k,i}^2}$.
     \label{line:master-score}
  \State $\cD_{k+1,h}\gets\cD_{k,h}\cup\{(\phi_{k,h},S_{k,h+1})\}$ for all $h\in[H]$ \label{line:master-history}
\EndFor
\end{algorithmic}
\end{algorithm}

\paragraph{Step 1. Construct a finite-support exploration distribution. }   

The continuous exponential-weight distribution over $\cO$ has density $q_k(x) \propto \exp(-\eta Z_{k-1}^\top x)\ind\{x\in\cO\}$. Let $\widehat\mu_k\in\R^p$ and
$\widehat\Sigma_k\in\R^{p\times p}$ be estimates of the mean and
covariance of $q_k$. Instead of sampling a point directly from $q_k$,
we construct a uniform distribution over the following $2p$ points:
\begin{align}
     \left\{ \widehat\mu_k + c_0 \widehat\Sigma_k^{1/2} v:~ v \in \{\pm e_1, \pm e_2, \ldots, \pm e_p\} \right\} \label{eq: finite support pints} 
\end{align}
where $e_1, \ldots, e_p$ are the standard basis vectors of $\R^p$, $c_0>0$ is a universal constant. This uniform distribution has mean $\widehat\mu_k$ and covariance $(c_0^2/p)\widehat\Sigma_k$. This sampling scheme is similar to that used in SCRiBLe \citep{abernethy2009competing}. 

We aim for finite support because Step~2 explicitly enumerates and processes every possible support point before the learner samples a policy. 

\paragraph{Step 2. Compile each support point or generate a cut. } 
Denote the $2p$ points in \eqref{eq: finite support pints} as $\{Y_{k,i}\}_{i=1}^{2p}$. For each $i\in[2p]$, \compile returns exactly one of the following:
\begin{itemize}[
    leftmargin=35pt,
    topsep=0pt,
    itemsep=0pt,
    parsep=0pt,
    partopsep=0pt
]
    \item[Case~1.] A set of policies $\{\pi_{k,i,j}\}_{j=1}^M$ and an error certificate $\varepsilon_{k,i}$ such that the average policy feature $\frac{1}{M}\sum_{j=1}^M\phi^{\pi_{k,i,j}}$ approximates $Y_{k,i}$, with approximation error quantified by $\varepsilon_{k,i}$.       
    \item[Case~2.] A cut $(c,b)$ such that $c^\top Y_{k,i}<b$ while, with high probability, $c^\top x \geq b$ for all $x\in \cK_\zeta$. 
\end{itemize}
If the return is Case~1 for all $2p$ points, the algorithm will proceed to Step~3; otherwise, the cut will be used to refine $\cO$, and the algorithm will go back to Step~1 to generate new points $\{Y_{k,i}\}_{i=1}^{2p}$ with the refined $\cO$. 
At a high level, \compile implements an \emph{approximate separation oracle} for $\cK_\zeta$. 
The approximation error $\varepsilon_{k,i}$ is constructed with confidence bounds from past
transition data. \compile returns a cut only when the confidence bounds certify that the queried point is outside $\cK_\zeta$. We will also show that: with high probability, the number of cuts produced by \compile is upper bounded by $\poly(p,K)$ during the learning procedure, so the algorithm will finally proceed to the next step.

\paragraph{Step 3. Sample a policy. } If \compile returns Case~1 for every $i\in[2p]$, the learner draws $(I_k,J_k)$ uniformly from $[2p]\times[M]$ and executes $\pi_{k,I_k,J_k}$. Conditional on the completed preparation and $I_k$, averaging the policy feature over $J_k$ gives $\frac1M\sum_{j=1}^M\phi^{\pi_{k,I_k,j}}$, which approximates $Y_{k,I_k}$ within $\varepsilon_{k,I_k}$ by the guarantee in Case~1. Therefore, the overall algorithm approximates a standard linear bandit algorithm with the sampling scheme \eqref{eq: finite support pints}. After execution, the constructed loss estimator $\widehat\theta_k$ and a bias-correction bonus $B_k$ are used to update~$Z_k$.

\section{Sketch of Regret Analysis and Bonus Design}
\label{sec:paper-algorithm}

Let $\mu_k=\E_{X\sim q_k}[X]$ and $\Sigma_k=\operatorname{Cov}_{X\sim q_k}(X)$
denote the exact mean and covariance of $q_k$ after preparation of episode $k$;
$\widehat\mu_k$ and $\widehat\Sigma_k$ are their estimates.
Write $\E_k[\cdot]=\E[\cdot\mid\cF_k]$, where $\cF_k$ is the information
available after preparation and before drawing $(I_k,J_k)$.

The key step is the \compile procedure in Step~2 in \cref{sec: algorithm overview}. In this section, we provide a skeleton of regret analysis. Then in \cref{sec:paper-compiler}, we provide detailed explanation for \compile. %

Recall that in episode $k$, \cref{alg:paper-core} forms $2p$ \emph{target} points (Line~\ref{line:master-targets} of \cref{alg:paper-core}). %
For each target $Y_{k,i}\in \R^p$, the \compile procedure (Line~\ref{line:reply compile} of \cref{alg:paper-core}) either returns a cut or a list of $M$ policies and a certificate $\varepsilon_{k,i}$. Let $x_{k,i}=\frac{1}{M}\sum_{j=1}^M\phi^{\pi_{k,i,j}}$. Then \cref{lem: approximate error lem} shows $\|x_{k,i} - Y_{k,i}\|_2\leq \varepsilon_{k,i}$. Further define $\xbar_k=\frac{1}{2p}\sum_{i=1}^{2p}x_{k,i}$.  Since $\widehat\mu_k=\frac{1}{2p}\sum_{i=1}^{2p}Y_{k,i}$, Jensen's inequality gives
\begin{align}
    \textstyle \norm{\xbar_{k} - \widehat\mu_k}_2 \leq \varepsilon_{k} := \sqrt{\frac1{2p}\sum_{i=1}^{2p}\varepsilon_{k,i}^2}.
    \label{eq: xY error}
\end{align}
Essentially, $\widehat\mu_k$ is the mean feature suggested by the underlying linear bandit algorithm, while $\xbar_{k}$ is the true expected feature of the policy executed by the learner. Their mismatch indicates that the underlying linear bandit algorithm is facing \emph{misspecified} or \emph{corrupted} feedback.

\paragraph{Quantifying bias and injecting bias-correction bonus.} To handle misspecification, we inject bias-correction bonus $B_k$ into the continuous exponential weight update  \citep{liu2024corruption,vanerven2025reduction}. To understand the design of $B_k$, we perform the following analysis. When the regret comparator has policy feature $u$, the true regret in episode $k$ is $\inner{\xbar_k-u, \theta_k}$, while the regret seen by the exponential weight algorithm is $\inner{\mu_k -u, \widehat\theta_k - B_k}$.  Their difference is
\begin{align}
    \underbrace{\inner{\xbar_k - u, \theta_k}}_{\text{true regret}} - \underbrace{\inner{\mu_k - u, \E_k[\widehat\theta_k]- B_k } }_{\text{perceived regret}} 
    &= \underbrace{\inner{\xbar_k - \mu_k, \theta_k}}_{\text{feature error}} + \underbrace{\inner{\mu_k - u, \theta_k - \E_k[\widehat\theta_k]}}_{\text{loss estimator bias}} + \underbrace{\inner{\mu_k-u, B_k}}_{\text{correction term}}.  \label{eq: bias decompose} 
\end{align}
The first term (i.e., feature error) comes from the mismatch between the estimated and true policy feature. It is upper bounded by $\sqrt{p}\big(\varepsilon_{k} + \|\widehat\mu_k - \mu_k\|_2\big)$ by \eqref{eq: xY error} with the assumption $\|\theta_k\|_2 \leq \sqrt{p}$. The term $\|\widehat\mu_k-\mu_k\|_2$ can be made negligible by running \moment with high enough accuracy. Therefore, this part is dominated by $\sqrt{p}\varepsilon_k$.  

The second term (i.e., loss estimator bias) comes from the bias of $\widehat\theta_k$. It can be upper bounded as $\inner{\mu_k - u, \theta_k - \E_k[\widehat\theta_k]}\leq \norm{\mu_k  - u}_{\Sigma_k^{-1}}\|\theta_k - \E_k[\widehat\theta_k]\|_{\Sigma_k}$. \Cref{lem:corrected-estimator} gives $\|\theta_k-\E_k[\widehat\theta_k]\|_{\Sigma_k}\leq \frac{5p\varepsilon_k}{2c_0}$, and hence
\begin{align}
   \textstyle \text{loss estimator bias}=\inner{\mu_k - u, \theta_k - \E_k[\widehat\theta_k]}\leq \frac{5p\varepsilon_k}{2c_0} \|\mu_k - u\|_{\Sigma_k^{-1}}.\label{eq: second bias}
\end{align}

 This bias depends on the regret comparator $u$ and can be large if distribution $q_k$ puts low weight on $u$.  
The solution to this term has become standard \citep{lee2020bias, zimmert2022, liu2024corruption}: 
The continuous exponential weight $q_k(x)\propto \exp(-\eta Z_{k-1}^\top x)\ind\{x\in\cO\}$ over $x\in\R^p$ can be viewed as Follow-the-Regularized-Leader with a $p$-self-concordant barrier \citep{bubeck2014entropic}. \cref{lem:entropic-toolkit} therefore gives
\(
 \norm{\mu_k-u}_{\Sigma_k^{-1}}
 \le-4\eta\langle\mu_k-u,Z_{k-1}\rangle+24p.
\)
Combining with \eqref{eq: second bias} and recalling
$B_k=\frac{10\eta p\varepsilon_k}{c_0}Z_{k-1}$, we can bound the sum of the last two terms in \eqref{eq: bias decompose} as
\begin{align*}
\textstyle \text{loss estimator bias} + \text{correction term}=
 \langle\mu_k-u,\theta_k-\E_k[\widehat\theta_k] + B_k\rangle
 \le \frac{60p^2\varepsilon_k}{c_0}.
\end{align*}
This explains why we include the bonus term $B_k$ in Line~\ref{line:master-score} of \cref{alg:paper-core}. The bonus removes the comparator-dependent term $O(p^2\varepsilon_k \|\mu_k-u\|_{\Sigma_k^{-1}})$ from the bias upper bound and introduces a comparator-independent overhead $O(p^2\varepsilon_k)$. Overall, the discrepancy between the true regret and the perceived regret is of order $O\big(p^2\sum_{k=1}^K\varepsilon_k\big).$

\paragraph{Bounding the perceived regret.} 
The perceived regret in \eqref{eq: bias decompose} follows the standard analysis of exponential weights. Even though the decision sets change over time, since they are always shrinking and the true regret comparator always lies in $\cK\subseteq \cO$ with high probability, the standard FTRL analysis applies (\cref{lem:nested-partition-bound}).

The true regret is then bounded by the sum of the perceived regret and the discrepancy: 

\begin{lemma}[Regret bound]
\label{prop:paper-master}
\cref{alg:paper-core} with the optimal choice of $\eta$ ensures 
\begin{equation}
\Reg_K\le
\tilde{\mathcal{O}}\bigg(\underbrace{d^{3/2}H^{5/2}\sqrt K}_{\textup{perceived regret}}
 +\underbrace{d^2H^2  \sqrt{K\E\left[\mathcal\sum_{k=1}^K \varepsilon_k^2\right]}}_{\textup{bias + correction}}\bigg). 
 \label{eq:paper-master-guarantee}
\end{equation}
\end{lemma}
To obtain $\widetilde O(\sqrt K)$ regret from \cref{eq:paper-master-guarantee}, it remains to bound $\sum_k \varepsilon_k^2$ by polynomial factors in $d,H$ and logarithmic factors in $K$. \Cref{sec:paper-compiler} establishes this guarantee, and we give the final bound for the ``bias + correction'' term above in \cref{sec: final analysis}. 

\FloatBarrier

%% file: sections/main_04_compiler.tex
\providecommand{\compactvspace}[1]{\vspace{#1}}

\section{From a target feature vector to a policy list}
\label{sec:paper-compiler}

This section builds \compile, the procedure in \cref{alg:paper-core} that performs
\emph{approximate separation}:
\compactvspace{-16pt}
\begin{figure}[H]
\centering
\setlength{\fboxsep}{5pt}
\fbox{%
  \parbox{\dimexpr\linewidth-2\fboxsep-2\fboxrule\relax}{%
    \textbf{Input}: $y\in\R^p$ \\[3pt]
    \textbf{Output}: exactly one of the following: 
    \begin{itemize}[leftmargin=13pt]
        \item $\big(\{\pi_j\}_{j=1}^M, \varepsilon\big)$ such that $\|y-\frac{1}{M}\sum_{j=1}^M \phi^{\pi_j}\|_2 \leq \varepsilon$. 
        \item $(c,b)$ such that $c^\top y < b$ and $c^\top x\geq b$ for all $x\in\cK_\zeta$.  
    \end{itemize}
  }%
}
\compactvspace{-10pt}
\caption{Approximate separation --- realized by \compile (\cref{alg:compiler})}
\label{fig:approx-separation-interface}
\end{figure}
\compactvspace{-12pt}
To build approximate separation, we leverage the ability to perform \emph{approximate linear
optimization}: 
\compactvspace{-9pt}
\begin{figure}[H]
\centering
\setlength{\fboxsep}{5pt}
\fbox{%
  \parbox{\dimexpr\linewidth-2\fboxsep-2\fboxrule\relax}{%
    \textbf{Input}: $c\in\R^p$ \\[3pt]
    \textbf{Output}: $(\pi, \phihat, \Delta)$ such that $\|\phihat - \phi^\pi\|_2\leq \Delta$ and  $c^\top \phihat \leq \min_{x\in \cK} c^\top x + \|c\|_2 \Delta$. 
  }%
}
\compactvspace{-10pt}
\caption{Approximate linear optimization --- realized by \certified (\cref{alg:lmo})}
\label{fig:approx-linear-optimization-interface}
\end{figure}
\compactvspace{-12pt}
In the analysis below, we will see how $\varepsilon$ in \cref{fig:approx-separation-interface} depends on $\Delta$ in \cref{fig:approx-linear-optimization-interface} (Line~\ref{line:compiler-return} of \cref{alg:compiler}). In implementing approximate linear optimization, the learner has access to previous transition data $(\cD_{k,h})_{h=1}^H$ as the source of information for the unknown set $\cK$. The error $\Delta$ will be related to confidence bound of the transition (see \eqref{eq:policy-certificate-bound}). 

Since $\pi$ is selected based on the same data used to estimate
its feature, we establish the confidence bounds underlying $\Delta$
uniformly over all possible output policies. We obtain these bounds by restricting
the policy class to have controlled complexity and applying a
covering argument, allowing the same data to be reused across
adaptive linear-optimization queries.

In \cref{sec: reduction}, we focus on building approximate separation from approximate linear optimization. Then in \cref{sec:paper-support-construction}, we investigate the implementation of approximate linear optimization.

\subsection{Constructing the policy list or excluding the target}\label{sec: reduction}

This subsection constructs \compile (\cref{alg:compiler}), assuming access to \certified. The
algorithm can be viewed as an inexact Frank--Wolfe procedure for the objective
$f(x)=\frac12\|x-y\|_2^2$ over $\cK$. As in standard Frank--Wolfe, the iterate is an average over extreme points. Thus, the procedure also naturally
produces a policy list whose average feature approximates the target $y$.  %

\begin{algorithm}[H]
\caption{\compile$(y,(\cD_{k,h})_{h=1}^H,M,\zeta)$}
\label{alg:paper-compiler}
\label{alg:compiler}
\begin{algorithmic}[1]
\Require Target $y\in\R^p$, transition data $(\cD_{k,h})_{h=1}^H$, policy count $M$, enlargement radius $\zeta$.
\State $\widehat x_0\gets0$.
\For{$j=1,\ldots,M$}
  \State $c_j\gets\widehat x_{j-1}-y$ \label{line:compiler-direction}
  \State $(\pi_j,\widehat\phi_j,\Delta_j)\gets\certified(c_j;(\cD_{k,h})_{h=1}^H)$ (\cref{alg:paper-support}) \label{line:compiler-support}
  \State $b_j\gets c_j^\top\widehat\phi_j-\norm{c_j}_2(\Delta_j+\zeta)$ \label{line:compiler-threshold}
  \If{$c_j^\top y<b_j$}
    \State \Return cut $(c_j,b_j)$ \label{line:compiler-cut}
  \EndIf
  \State $\widehat x_j\gets\frac{j-1}{j}\widehat x_{j-1}+\frac1j\widehat\phi_j$ \label{line:compiler-average}
\EndFor
\State \Return $((\pi_j)_{j=1}^M,\varepsilon)$, where $\varepsilon=\sqrt{4\zeta^2+\frac6M\sum_{j=1}^M\Delta_j^2+\frac{8R^2(1+\log M)}M}$ \label{line:compiler-return}
\end{algorithmic}
\end{algorithm}

First, we ensure that if \certified performs as desired, then \compile returns a cut only when $y\notin \cK_\zeta$. This is shown in the following lemma: 
\begin{lemma}\label{lem: first guar}
    Suppose \certified (called in Line~\ref{line:compiler-support} of \cref{alg:compiler}) satisfies the property in \cref{fig:approx-linear-optimization-interface} and assume $y\in \cK_\zeta$. Then \compile will NOT return a cut. 
\end{lemma}
\begin{proof}
    By the guarantee of \certified, we have 
    \begin{align*}
        \textstyle c_j^\top \phihat_j &\leq \min_{x\in\cK} c_j^\top x + \|c_j\|_2 \Delta_j \leq \big(\min_{x\in\cK_\zeta} c_j^\top x + \zeta\|c_j\|_2\big) +\|c_j\|_2 \Delta_j   \leq c_j^\top y + \|c_j\|_2 (\Delta_j +\zeta)   
    \end{align*}
    where we use the assumption $y\in\cK_\zeta$ in the last inequality. Thus, the ``if'' condition in Line~\ref{line:compiler-cut} of \cref{alg:compiler} does not hold, and the algorithm will not return a cut. 
\end{proof}

Next, we ensure that if no cut is returned by \compile, the approximation error for $y$ is controlled. 
\begin{lemma}\label{lem: approximate error lem}
   Suppose \certified (called in Line~\ref{line:compiler-support} of \cref{alg:compiler}) satisfies the property in \cref{fig:approx-linear-optimization-interface} and suppose \compile does NOT return a cut. Then $\big\|\frac{1}{M}\sum_{j=1}^M \phi^{\pi_j} - y \big\|_2 \leq \varepsilon$ (defined in Line~\ref{line:compiler-return}).  
\end{lemma}
\begin{proof}
This proof is related to the analysis for Frank-Wolfe algorithms. For any $j=1,\ldots,M$, 
    \begin{align*}
        \|\xhat_j - y\|^2 
        &= \norm{\left(1-\frac{1}{j}\right)\xhat_{j-1} + \frac{1}{j}\phihat_j - y}^2   \tag{Line~\ref{line:compiler-average}}\\ 
        &= \norm{\xhat_{j-1} - y}^2 -\frac{2}{j} \inner{\xhat_{j-1}-y, \xhat_{j-1}-\phihat_{j}} + \frac{1}{j^2} \|\xhat_{j-1} - \phihat_j\|^2 \\ 
        &\leq  \norm{\xhat_{j-1} - y}^2 -\frac{2}{j} \|\xhat_{j-1} - y\|^2 - \frac{2}{j}\inner{c_j,y - \phihat_j}   + \frac{4R^2}{j^2} \tag{$c_j=\xhat_{j-1}-y$} \\ 
        &\leq  \norm{\xhat_{j-1} - y}^2 -\frac{2}{j} \|\xhat_{j-1} - y\|^2 + \frac{2}{j}\|c_j\|(\Delta_j + \zeta) + \frac{4R^2}{j^2} \tag{the ``if'' condition in Line~\ref{line:compiler-cut} of \cref{alg:compiler} does not hold}\\ 
        &\leq \left(1-\frac{1}{j}\right) \norm{\xhat_{j-1} - y}^2 + \frac{2\Delta_j^2+2\zeta^2}{j} + \frac{4R^2}{j^2}.  \tag{$c_j=\xhat_{j-1}-y$ and AM-GM}
    \end{align*}
Expanding the recursion using \cref{lem:averaging-recursion} yields
\begin{align}
 \norm{\xhat_M-y}_2^2
 \le 2\zeta^2 + \frac2M\sum_{j=1}^M\Delta_j^2
     +\frac{4R^2(1+\log M)}M. 
 \label{eq:main-compiler-estimated-average}
\end{align}

Let $x=\frac{1}{M}\sum_{j=1}^M\phi^{\pi_j}$. Since
$\xhat_M=\frac{1}{M}\sum_{j=1}^M\phihat_j$, by the assumed property of $\certified$, 
\begin{align}
 \norm{x-\xhat_M}_2^2
 \le\frac1M\sum_{j=1}^M
       \norm{\phi^{\pi_j}-\phihat_j}_2^2
 \le\frac1M\sum_{j=1}^M\Delta_j^2.   \label{eq: the other term}
\end{align}
Combining \eqref{eq: the other term} and \eqref{eq:main-compiler-estimated-average} and
using $\|x-y\|_2^2 \leq 2\|x-\xhat_M\|_2^2 + 2\|\xhat_M-y\|_2^2$ finishes the proof. 
\end{proof}

Together, the two lemmas establish the desired dichotomy: feasible targets are never excluded, while any target not excluded is approximated by a policy list with controlled error. The next subsection relates this error to confidence bounds for the unknown transitions.

\subsection{Planning and evaluating from past transitions}
\label{sec:paper-support-construction}

In this subsection, we develop \certified that achieves the property in \cref{fig:approx-linear-optimization-interface}, with the help of the transition data $(\cD_{k,h})_{h=1}^H$ collected before. The algorithm is adapted from the standard Least-Square Value Iteration (LSVI) procedure \citep{jin2020, wang2020reward} and is displayed in \cref{alg:lmo}. Its guarantee is summarized in the following theorem:

\begin{theorem}[Certified approximate linear optimization]
\label{thm:certified-lmo}
Let $\beta = C_\beta d\sqrt{\log(KdHA/\delta)}$, 
where $C_\beta$ is a sufficiently large universal constant. Then, with
probability at least $1-\delta$, simultaneously for every
episode $k$ and every call to \cref{alg:lmo}, the returned tuple
$(\pi,\widehat\phi,\Delta)$ satisfies
$\widehat\phi\in\sqrt H\Ball_2^p$, $\Delta\in[0,2\sqrt H]$, and
\begin{equation}
\textstyle  \|\widehat\phi-\phi^\pi\|_2\le\Delta,
 \qquad
 c^\top\widehat\phi-\norm c_2\Delta
 \le\min_{u\in\cK}c^\top u.
 \label{eq:two-bound}
\end{equation}
Moreover, the returned certificate satisfies
\begin{equation}
 \Delta
 \le10\sqrt p\,\beta^2H\,
 \E_{P,\pi}\!\left[
   \sum_{h=1}^H
   \norm{\phi_h(S_h,A_h)}_{\Lambda_{k,h}^{-1}}
 \right]
 +H\tau\log A.
 \label{eq:policy-certificate-bound}
\end{equation}
\end{theorem}

\begin{algorithm}[t]
\caption{\certified$(c,(\cD_{k,h})_{h=1}^H)$}
\label{alg:paper-support}
\label{alg:lmo}
\begin{algorithmic}[1]
\Require $c\in\R^p$, dataset $(\cD_{k,h})_{h=1}^H$. 
\Parameters $\tau=1/(H\sqrt K\log(eA))$ and $\beta$ as in \cref{thm:certified-lmo}.
\vspace{3pt}
\Define $\Lambda_{k,h}=\bI_d+\sum_{(u,s')\in\cD_{k,h}}uu^\top$, $\nu_{k,h}(s,a) = \|\phi(s,a)\|_{\Lambda_{k,h}^{-1}}$, and $\bar c= c/\norm c_2$.   
\compactvspace{-1pt}
\Define $\widehat P_{k,h}f(s,a)=\phi_h(s,a)^\top\Lambda_{k,h}^{-1}
       \sum_{(u,s')\in\cD_{k,h}}u f(s')$ for $f: \cS\to \R$. 
\vspace{3pt}
\State Let $V_{H+1}^{\pl}(s)=0\in \R$ and $V_{H+1}^{\feat}(s)=0\in\R^p$ and $\widehat U_{H+1}(s)=0\in\R$ for all $s\in\cS_{H+1}$. 
\vspace{5pt}
\For{$h=H,H-1,\ldots,1$}
  \vspace{7pt}
  \State {\color{blue}\textsf{// Task 1: policy optimization: choose the policy $\pi$}}
  \vspace{2pt}
  \State $Q_h^{\pl}(s,a)\gets\CLIP_{[-\sqrt H,\sqrt H]}\big(\phi_h(s,a)^\top \bar{c}_h+\widehat P_{k,h}V_{h+1}^{\pl}(s,a)-(\sqrt{H}+1)\beta \nu_{k,h}(s,a)\big)$. \label{line:support-planning}
  \State $V_h^{\pl}(s)\gets-\tau\log\sum_a \exp(-Q_h^{\pl}(s,a)/\tau)$.  \label{line:support-softmin}
  \State $\pi_h(a\mid s)\propto \exp(-Q_h^{\pl}(s,a)/\tau)$. \label{line:support-policy}
  \vspace{7pt}
  \State {\color{blue}\textsf{// Task 2: feature evaluation: estimate $\phi^\pi$}}
  \vspace{2pt}
  \State $r_h^{\feat}(s,a)\gets(0,\ldots,0,\phi_h(s,a),0,\ldots,0)\in\R^p$. \label{line: rft definition}  %
  \State $Q_h^{\feat}(s,a)\gets\CLIP_{[-1,1]^p}(r_h^{\feat}(s,a)+\widehat P_{k,h}V_{h+1}^{\feat}(s,a))$ \label{line:support-feature-backup}
  \State $V_h^{\feat}(s)\gets\sum_a\pi_h(a\mid s)Q_h^{\feat}(s,a)$ \label{line:support-feature-average}
  \vspace{7pt}
  \State {\color{blue}\textsf{// Task 3: uncertainty evaluation: upper bound $U^\pi_k$}}
  \vspace{2pt}
  \State $Q_h^{\rm unc}\gets\CLIP_{[0,H-h+1]}( (1+H\beta) \nu_{k,h}+\widehat P_{k,h}\widehat U_{h+1})$ \label{line:support-uncertainty-backup}
  \State $\widehat U_h(s)\gets\sum_a\pi_h(a\mid s)Q_h^{\rm unc}(s,a)$ \label{line:support-uncertainty-average}
\EndFor
\State $\widehat\phi\gets\Proj_{\sqrt H\Ball_2^p}(V_1^{\feat}(s_1))$. 
\State $\Delta\gets\min\{5\sqrt p\beta\widehat U_1(s_1)+H\tau\log A,2\sqrt H\}$ \label{line:support-certificate}
\State \Return $(\pi,\widehat\phi,\Delta)$ \label{line:support-save}
\end{algorithmic}
\end{algorithm}

Below, we overview the elements in \cref{alg:lmo}. %

\paragraph{Task 1: Policy optimization.}
We use optimistic LSVI \citep{jin2020} to find an approximately optimal policy under the cost function $\phi_h(s,a)^\top c_h$. Different from \cite{jin2020}, we output a Gibbs (softmin) policy (Line~\ref{line:support-policy} of \cref{alg:lmo}) instead of a greedy one, and use soft state value (Line~\ref{line:support-softmin}). The reason of doing this is to make the mapping from $c$ to the output policy feature $\phihat$ be Lipschitz, leading to a bounded log covering number for the class of mappings (\hyperref[app:confidence]{Appendix~\ref*{app:confidence}}). The clipping for $Q_h^\pl$ is at $\sqrt{H}$ because $|\inner{\phi(s,a), \bar{c}}|\leq \|\phi(s,a)\|_2\|\bar{c}\|_2 \leq \sqrt{H}$ (Line~\ref{line:support-planning}). This implies $|V_h^\pl(s)|\leq \sqrt{H} + \tau \log A\leq \sqrt{H}+1$.  

\mbox{The soft value preserves optimism because $V_h^{\pl}(s)\le\min_a Q_h^{\pl}(s,a)$, giving $V_1^{\pl}(s_1)\le\min_{u\in\cK}\bar c^\top u$.} Thus, to obtain the second guarantee in \eqref{eq:two-bound}, it suffices to bound $\bar c^\top\widehat\phi-V_1^{\pl}(s_1)$. This difference can be decomposed into 
\begin{align}
    \bar c^\top\widehat\phi-V_1^{\pl}(s_1) = \underbrace{\bar c^\top(\widehat\phi-\phi^\pi)}_{\text{feature estimation error}} + \underbrace{\bar c^\top \phi^\pi-V_1^{\pl}(s_1)}_{\text{planning error}} \label{eq: loss decom}
\end{align}
The feature-evaluation error is handled in the feature-evaluation step below. The planning error comes from Gibbs smoothing and the transition-prediction errors, which sum to
\(
 \bar c^\top \phi^\pi-V_1^{\pl}(s_1)
 \le2(\sqrt H+1)\beta U_k^\pi+H\tau\log A
\), where $U_k^\pi=\E_{P,\pi}[\sum_h\nu_{k,h}(S_h,A_h)]$. 

\paragraph{Task 2: Feature evaluation.}

Now fix the Gibbs policy $\pi$ constructed in the policy-optimization step, and estimate its policy feature $\phi^\pi$. Define vector $r_h^{\feat}(s,a)\in \R^p$ as in Line~\ref{line: rft definition}.  %
Then we have 
\[
 \E_{P,\pi}\!\left[\sum_{h=1}^H r_h^{\feat}(S_h,A_h)\right]
 =\bigl(
   \E_{P,\pi}[\phi_1(S_1,A_1)],\ldots,
   \E_{P,\pi}[\phi_H(S_H,A_H)]\bigr)
 =\phi^\pi.
\]
Thus estimating $\phi^\pi$ is a policy-evaluation problem with $p$-dimensional \emph{vector rewards}. We estimate $\phi^\pi$ by backward policy evaluation, obtaining $\widehat\phi\in\sqrt H\Ball_2^p$. A backward induction shows that the estimation error is controlled by $U_k^\pi$, giving
\(\|\widehat\phi-\phi^\pi\|_2\le\sqrt p\,\beta U_k^\pi.
\) Since $\norm{\bar c}_2\le1$, $\big|\bar c^\top\widehat\phi - \bar c^\top \phi^\pi \big| \le \sqrt p\,\beta U_k^\pi$. Adding this estimation error to the policy optimization bound \eqref{eq: loss decom} gives
\begin{align}
 \bar c^\top\widehat\phi-\big[(\sqrt p + 2\sqrt{H} + 2)\beta U_k^\pi+H\tau \log A\big]
 \le V_1^{\pl}(s_1)\le\min_{u\in\cK}\bar c^\top u.  \label{eq: delta tmp def}
\end{align}
Thus both inequalities in \eqref{eq:two-bound} hold if we choose $\Delta\ge 5\sqrt p\beta U_k^\pi+H\tau \log A$. However, the learner cannot compute this threshold because $U_k^\pi$ is an expectation under the unknown transitions. We next obtain an estimated upper bound of $U_k^\pi$, from which we can choose $\Delta$.

\paragraph{Task 3: Uncertainty evaluation.}
We use the same backward policy evaluation to estimate $U_k^\pi$. This can be achieved by viewing $\nu_{k,h}(s,a)\in[0,1]$ as the reward. Since $\|\widehat U_{h+1}\|_\infty\le H$, we add the bonus $(1+\beta H)\nu_{k,h}$ to guard against underestimating future uncertainty. A backward induction gives
\(
 U_k^\pi\le\widehat U_1(s_1)\le2\beta H U_k^\pi.
\)
Using the computable upper bound $\widehat U_1(s_1)$ in place of $U_k^\pi$ in the feature evaluation step, \certified$(c;(\cD_{k,h})_{h=1}^H)$ (\cref{alg:paper-support}) returns $(\pi,\widehat\phi,\Delta)$ with
\(
 \Delta=\min\{5\sqrt p\,\beta\widehat U_1(s_1)+H\tau \log A,2\sqrt H\}, 
\)
where the clipping to $2\sqrt{H}$ is because $\|\phi-\phi'\|_2\leq 2\sqrt{H}$ for any policy features $\phi, \phi'$. %

The above definition of $\Delta$ together with $\widehat U_1(s_1)\le2\beta H U_k^{\pi}$ gives
\eqref{eq:policy-certificate-bound}.  %

\subsection{Final Regret Bound} \label{sec: final analysis}
\begin{proof}[Proof of \cref{thm:paper-main}]
Now we combine elements to establish the final regret bound. The sum of misspecification in the ``bias + correction'' term in \cref{prop:paper-master} can be bounded as 
\begin{align*}
    &\E\left[\sum_{k=1}^K \frac{1}{2p}\sum_{i=1}^{2p}  \varepsilon_{k,i}^2\right]  \\ 
    &= \E\left[\sum_{k=1}^K  \left(4\zeta^2 + \frac{8R^2(1+\log M)}{M} + \frac{6}{M}\frac{1}{2p}\sum_{i=1}^{2p}  \sum_{j=1}^M\Delta_{k,i,j}^2\right)\right]  \tag{Line~\ref{line:compiler-return} of \cref{alg:compiler}} \\ 
    &\leq \widetilde{O}\left(H + \E\left[ \sum_{k=1}^K \Delta_{k,I_k, J_k}^2 \right]  \right)  \tag{$(I_k, J_k)$ uniformly drawn from $[2p]\times [M]$ in \cref{alg:paper-core}}\\ 
    &\leq \widetilde{O}\left(H + \E\left[ \sum_{k=1}^K p\beta^4 H^3 \sum_{h=1}^H  \|\phi_h(S_{k,h}, A_{k,h})\|_{\Lambda_{k,h}^{-1}}^2 \right]   \right) \tag{\cref{thm:certified-lmo}} \\ 
    &\leq \widetilde{O}\left(d^6H^5 \right),  \tag{Elliptical potential lemma: $\E\Big[\sum_{k,h} \|\phi_h(S_{k,h}, A_{k,h})\|_{\Lambda_{k,h}^{-1}}^2\Big]\leq\widetilde{O}(dH)$}
\end{align*}
Plugging this in \cref{prop:paper-master} gives the desired bound. 
\end{proof}

%% file: sections/main_06_computation.tex
\section{Computation}
\label{sec:paper-computation}

We bound the computation time of \cref{alg:paper-core} by analyzing
the costs of \Alg{MomentEstimate} and \compile, together with
the number of cuts.

The \Alg{MomentEstimate} procedure and the desired guarantees are same as those in \citep{ito2019oracle}. We directly use the results in their Section~4.2 to bound the complexity of \Alg{MomentEstimate} by $\poly(d,H,K,\log(1/\delta))$ calls to the membership oracle of $\cO$. 

The total number of cuts added in each episode is upper bounded by
$\widetilde O(p+\sqrt H K^{3/2})$ with high probability (see
\cref{lem:local-cut-progress}). To prevent the low probability event to
lead to non-termination of the algorithm, we present a version of
\cref{alg:paper-core} that explicitly counts the number of restarts and
executes the fallback plan if the count exceeds a certain threshold
(\cref{alg:master another}).

The \certified procedure (\cref{alg:lmo}) can be run in $\poly(d,H,K,A)$ time, even though the displayed \cref{alg:lmo} enumerates the updates on all state-action pairs. This is common in prior work for linear MDP (\cite{jin2020}). For completeness, we explicitly describe the implementation in \cref{app:complexity}.

%% file: appendices/contents.tex
\begingroup
\renewcommand{\contentsname}{Appendix contents}
\hypersetup{linktoc=all}
\tableofcontents
\endgroup
\clearpage

%% file: appendices/related_work.tex
\section{Related Work}
\label{app:related-work}

\subsection{Learning Adversarial MDPs with Linear Structure}
\label{app:related-linear-mdps}

We review results for learning adversarial rewards with fixed transition dynamics.

A closely related problem is adversarial linear contextual bandits with stochastic contexts, where contexts are drawn independently from a fixed distribution, while reward vectors may change adversarially. \citet{neu2020contextual} obtain efficient square-root regret when the context distribution is known. With access to a context simulator, \citet{dai2023refined} also achieve square-root regret without knowing this distribution. \citet{liu2023simulator} remove the simulator requirement and obtain $\sqrt K$ regret with computation polynomial in the number of available actions. \citet{vanerven2026reduction} obtains a reduction to linear bandits whose computation avoids enumerating the actions, and only requires linear optimization oracles for the action sets.

Linear mixture MDPs parameterize the transition kernel using known features of state-action-next-state triples, rather than the state-action features in linear MDPs. With unknown transitions and full-information reward feedback, \citet{cai2020oppo} establish $\sqrt K$ regret via optimistic policy optimization. \citet{he2022mixture} improve the horizon dependence and obtain near-minimax guarantees. Under per-step bandit feedback, \citet{zhao2023mixture} give the first efficient algorithm with $\sqrt K$ regret, and \citet{li2024mixture} subsequently improve its dependence on the number of states. Both bandit-feedback bounds retain polynomial dependence on state cardinality.

For adversarial linear MDPs, \citet{neu2021mdp} establish square-root regret with known transitions under an exploratory condition. With unknown transitions and full-information feedback, \citet{zhong2023} obtain $K^{3/4}$ and \citet{sherman2024optimal} improved the guarantee to $\sqrt K$. \citet{casselrosenberg2024} retain the square-root rate, improve the dimension and horizon dependencies, and remove the initial reward-free exploration phase in \citet{sherman2024optimal}. For bandit-feedback under a simulator access, \citet{luo2021policy} obtain $K^{2/3}$ regret, and \citet{dai2023refined} improve this to $\sqrt K$.  Without a simulator, learning unknown transitions from the learner's own trajectories is more difficult. Under bandit feedback, the efficient regret bounds progress from $K^{14/15}$ \citep{luo2021policy} to $K^{8/9}$ \citep{dai2023refined}, $K^{6/7}$ \citep{sherman2023improved}, and $K^{3/4}$ \citep{liu2024linear}. \citet{kong2024improved} obtain $K^{4/5}$ regret through a reduction to linear optimization under an exploratory assumption. \citet{liu2024linear} additionally achieve the optimal $\sqrt K$ rate without that assumption, but their algorithm is computationally inefficient. Our work obtains $\poly(d,H)\sqrt K$ regret with polynomial computation, and requires only the cummulative loss rather than the per-step losses along the trajectory.

\subsection{Adversarial RL beyond Linear MDPs}
\label{app:related-representations}

Low-rank MDPs extend linear MDPs by allowing the feature representation to be unknown. With unknown transitions and full-information feedback, \citet{zhao2023lowrank} and \citet{liu2024lowrank} give model-based, oracle-efficient algorithms with $K^{5/6}$ and $K^{2/3}$ regret, respectively, allowing arbitrary rewards. Under bandit feedback, \citet{liu2024lowrank} show that additional reward structure is necessary to avoid polynomial dependence on the number of states. Assuming linear rewards, they obtain $K^{2/3}$ regret through computationally inefficient model-based and model-free algorithms, and $K^{4/5}$ regret through oracle-efficient model-free algorithms.

The decision-estimation coefficient (DEC) approach provides a general approach to adversarial RL. \citet{foster2022adversarial, xu2025bayesian} characterize adversarial decision making using the DEC of the convexified model class, but they can only handle MDPs with known transitions. 
\citet{liu2025hybrid} extend the DEC framework to MDPs with fixed unknown transitions and adversarial rewards, providing general model-based regret guarantees. As an application, their framework yields $\sqrt K$ regret for low-rank MDPs with adversarial linear rewards under both full-information and bandit feedback.  For model-free learning with full-information reward feedback, they establish $K^{3/4}$ and $K^{5/6}$ regret for on-policy and off-policy bilinear classes \citep{du2021bilinear}, respectively.

\citet{liu2026digdec} extends model-free learning to adversarial bandit feedback under linear rewards. For on-policy and off-policy bilinear classes, they obtain $K^{5/6}$ and $K^{13/15}$ regret, respectively. With Bellman completeness, these rates improve to $K^{3/4}$ and $K^{4/5}$; they also obtain $K^{3/4}$ regret for coverable MDPs \citep{xie2022role} under Bellman completeness. These results cover transition structures beyond linear MDPs, but do not provide the polynomial-time square-root guarantee studied here.

\clearpage

%% file: appendices/03_master_algorithm.tex
\section{Complete Algorithm}
\label{sec:master-algorithm}
\label{sec:master-program}

\cref{alg:paper-core} involves a repeat-until loop that may not stop, though this happens only with negligible probability (\cref{lem:local-cut-progress}). In order to ensure the algorithm will terminate with probability $1$, we present below in \cref{alg:master another} a more complete version that caps the number of repeats by a finite amount. As shown in \cref{lem:local-cut-progress}, with high probability, the number of repeats is smaller than the threshold. In the regret analysis (\cref{prop:paper-master}), the bad event only contributes to negligible expected regret. In \cref{alg:master another}, we highlight the parts that are different from \cref{alg:paper-core}. 

\begin{algorithm}[!ht]
\caption{\cref{alg:paper-core} with a restart cap and a fallback policy}
\label{alg:master another}
\begin{algorithmic}[1]
\Require $c_0=\frac{1}{16e}$, $\zeta=K^{-1/2}$,
$R=\sqrt H+\zeta$, $M=K$,
\textcolor{magenta}{$N_{\max}=\widetilde{\Theta}(p+\sqrt{HK^3})$},
$\eta\leq \frac{c_0}{128\sqrt{d^3H^4}}$.
\State $Z_0\gets0$; $\cO\gets R\Ball_2^p$;
$\cD_{1,h}\gets\varnothing$ for all $h\in[H]$
\For{$k=1,\ldots,K$}
  \State \textcolor{magenta}{$n\gets0$}
  \Repeat
    \State \textcolor{magenta}{$n\gets n+1$}
    \State \textit{restart}$\gets\Alg{false}$
    \State Define
    $q_k(x)\propto\exp(-\eta Z_{k-1}^\top x)\ind\{x\in\cO\}$
    \State $(\widehat\mu_k,\widehat\Sigma_k)
    \gets\Alg{MomentEstimate}(\cO,Z_{k-1})$
    \State Let $\{Y_{k,i}\}_{i=1}^{2p}$ be the set
    $
      \left\{
      \widehat\mu_k+c_0\widehat\Sigma_k^{1/2}v:
      v\in\{\pm e_1,\pm e_2,\ldots,\pm e_p\}
      \right\}.
    $
    \For{$i=1,\ldots,2p$}
      \State $\mathsf{reply}\gets
      \compile(Y_{k,i},(\cD_{k,h})_{h=1}^H,M,\zeta)$
      \If{$\mathsf{reply}$ is a cut $(c,b)$}
        \State $\cO\gets\cO\cap\{x:c^\top x\ge b\}$
        \State \textit{restart}$\gets\Alg{true}$; \textbf{break}
      \Else
        \State $((\pi_{k,i,j})_{j=1}^M,\varepsilon_{k,i})
        \gets\mathsf{reply}$
      \EndIf
    \EndFor
  \Until{\textit{restart}$=\Alg{false}$
  \textcolor{magenta}{or $n=N_{\max}$}}
  \color{magenta}
  \If{$n=N_{\max}$}
    \State Execute the policy $\pi_k$ with
    $\pi_{k,h}(\cdot\mid s)=\operatorname{Unif}(\cA)$ for all $h,s$.
    \State $Z_k\gets Z_{k-1}$ and
    $\cD_{k+1,h}\gets\cD_{k,h}$ for all $h\in[H]$
  \Else
    \color{black}
    \State Draw $(I_k,J_k)\sim\operatorname{Unif}([2p]\times[M])$
    and execute $\pi_k=\pi_{k,I_k,J_k}$.
    \State Observe feature vectors along the trajectory
    $(\phi_{k,h})_{h=1}^H$ and the total loss $L_k$.
    \State $\widehat\theta_k\gets
    \left(\zeta^2\mathbf{I}_p+\frac{c_0^2}{p}\widehat\Sigma_k\right)^{-1}
    (Y_{k,I_k}-\widehat\mu_k)L_k$
    \State $Z_k\gets Z_{k-1}+\widehat\theta_k-B_k$, where
    $B_k=\frac{10\eta p\varepsilon_k}{c_0}Z_{k-1}$ and
    $\varepsilon_k=\sqrt{\frac1{2p}\sum_{i=1}^{2p}\varepsilon_{k,i}^2}$.
    \State $\cD_{k+1,h}\gets\cD_{k,h}
    \cup\{(\phi_{k,h},S_{k,h+1})\}$ for all $h\in[H]$
    \color{magenta}
  \EndIf
  \color{black}
\EndFor
\end{algorithmic}
\end{algorithm}

%% file: appendices/E_outer_details.tex
\newpage
\section{Omitted Proofs in \cref{sec:paper-algorithm}}
\label{app:outer-details}

\subsection{Properties of continuous exponential weights}

For a full-dimensional compact convex body $\cO\subset\R^p$, define
\begin{equation}
  \Phi_{\cO}(v)=\log\int_{\cO}e^{v^\top x}\,dx,
  \qquad
  F_{\cO}=\Phi_{\cO}^*.
  \label{eq:entropic-barrier-def}
\end{equation}

Here $F_{\cO}(x)=\sup_v\{v^\top x-\Phi_{\cO}(v)\}$ is the entropic barrier.

\begin{lemma}[Entropic-barrier toolkit]
\label{lem:entropic-toolkit}
Let $X_v$ have density proportional to $e^{v^\top x}\ind\{x\in\cO\}$.  Then
\begin{equation}
  \nabla\Phi_{\cO}(v)=\E\left[X_v\right],
  \qquad
  \nabla^2\Phi_{\cO}(v)=\operatorname{Cov}(X_v).
  \label{eq:partition-derivatives}
\end{equation}
At every conjugate pair $x=\nabla\Phi_{\cO}(v)$,
\begin{equation}
  \nabla F_{\cO}(x)=v,
  \qquad
  \nabla^2F_{\cO}(x)=\bigl(\nabla^2\Phi_{\cO}(v)\bigr)^{-1}.
  \label{eq:barrier-duality}
\end{equation}
Moreover, for every $x,u\in\operatorname{int}(\cO)$,
\begin{align}
  \norm{x-u}_{\nabla^2F_{\cO}(x)}
  &\le4\ip{x-u}{\nabla F_{\cO}(x)}+24p,
  \label{eq:barrier-distance-comparison}\\
  \norm{\nabla F_{\cO}(x)}_{(\nabla^2F_{\cO}(x))^{-1}}
  &\le\sqrt p.
  \label{eq:barrier-gradient-bound}
\end{align}
Finally, if $r=\norm{\delta v}_{\nabla^2\Phi_{\cO}(v)}<1$, then
\begin{equation}
  \Phi_{\cO}(v+\delta v)-\Phi_{\cO}(v)
  -\ip{\nabla\Phi_{\cO}(v)}{\delta v}
  \le-r-\log(1-r).
  \label{eq:selfconcordant-taylor}
\end{equation}
\end{lemma}

\begin{proof}
\Cref{eq:partition-derivatives} follows by differentiation. Inverse-gradient duality gives \eqref{eq:barrier-duality} \citep[Section~2]{chewi2023}.

By \citet[Theorem~2]{chewi2023}, $F_{\cO}$ is a $p$-self-concordant barrier, giving \eqref{eq:barrier-gradient-bound}. By \citet[Lemma~14]{zimmert2022}, some $\gamma'\in[1,4]$ satisfies
\[
  \norm{x-u}_{\nabla^2F_{\cO}(x)}
  \le
  \gamma'\ip{x-u}{\nabla F_{\cO}(x)}+4\gamma'p+2\sqrt p.
\]
For $p,\gamma'\ge1$, $4\gamma'p+2\sqrt p\le6\gamma'p$. Thus
\[
 0\le\norm{x-u}_{\nabla^2F_{\cO}(x)}
 \le\gamma'\bigl(\ip{x-u}{\nabla F_{\cO}(x)}+6p\bigr)
 \le4\bigl(\ip{x-u}{\nabla F_{\cO}(x)}+6p\bigr),
\]
proving \eqref{eq:barrier-distance-comparison}.

Since $\Phi_{\cO}=F_{\cO}^*$ and self-concordance is preserved under
convex conjugation \citep[Section~2.1.3]{nemirovski2008},
$\Phi_{\cO}$ is self-concordant. Its local Taylor bound
\citep[Section~2.1.1, Eq.~(2.5)]{nemirovski2008} gives
\eqref{eq:selfconcordant-taylor}.
\end{proof}

\subsection{Bounding bias and bonus}
\label{sec:one-round-estimation}

\begin{lemma}[Loss-estimator bias and local norm]
\label{lem:corrected-estimator}
Assume $\frac12\Sigma_k\preceq\widehat\Sigma_k\preceq2\Sigma_k$, and write $\widetilde\theta_k = \widehat\theta_k - B_k$. 
Then the loss estimator satisfies
\begin{equation}
  \|\theta_k-\E_k[\widehat\theta_k]\|_{\Sigma_k}
  \le\frac{5p\varepsilon_k}{2c_0},
  \qquad
  \|\widehat\theta_k\|_{\Sigma_k}^2
  \le\frac{2H^2p^2}{c_0^2}, \qquad \|\widetilde\theta_k\|_{\Sigma_k}
  \le \frac{\sqrt2Hp}{c_0}
       +\frac{10p^{3/2}\varepsilon_k}{c_0}. 
  \label{eq:robust-bias-leverage-detail}
\end{equation}
\end{lemma}

\begin{proof}[Proof of \cref{lem:corrected-estimator}]
Define $\alpha=c_0^2/p$. 
From the estimator definition,
\begin{align*}
 \theta_k-\widehat\theta_k
 ={}&(\zeta^2\mathbf{I}_p+\alpha\widehat\Sigma_k)^{-1}
 \Bigl((\zeta^2\mathbf{I}_p+\alpha\widehat\Sigma_k)\theta_k
       -(Y_{k,I_k}-\widehat\mu_k)L_k\Bigr).
\end{align*}
Our construction gives
\begin{equation*}
 \frac1{2p}\sum_{i=1}^{2p}(Y_{k,i}-\widehat\mu_k)=0,
 \qquad
 \frac1{2p}\sum_{i=1}^{2p}
 (Y_{k,i}-\widehat\mu_k)(Y_{k,i}-\widehat\mu_k)^\top
 =\alpha\widehat\Sigma_k. 
\end{equation*}
Define
\begin{equation*}
 m_k:=\frac1{2p}\sum_{i=1}^{2p}
 (Y_{k,i}-\widehat\mu_k)(x_{k,i}-Y_{k,i})^\top\theta_k.
\end{equation*}
Using
$\E_k[L_k\mid I_k=i]=\langle x_{k,i},\theta_k\rangle$, we obtain
\begin{align*}
 \E_k[(Y_{k,I_k}-\widehat\mu_k)L_k]  
 &=\frac1{2p}\sum_{i=1}^{2p}
   (Y_{k,i}-\widehat\mu_k)x_{k,i}^\top\theta_k\\ 
   &=\alpha\widehat\Sigma_k\theta_k+m_k. 
\end{align*}
Taking $\E_k$ in the first display therefore gives
\begin{equation}
 \theta_k-\E_k[\widehat\theta_k]
 = (\zeta^2\mathbf{I}_p+\alpha\widehat\Sigma_k)^{-1}
   (\zeta^2\theta_k-m_k). \label{eq: bias deom}
\end{equation}

For any $v\in\R^p$, Cauchy--Schwarz for the uniform average over
$i\in[2p]$ gives
\begin{align*}
    \left|\inner{v,m_k}\right| &= \left|\frac1{2p}\sum_{i=1}^{2p}
 v^\top (Y_{k,i}-\widehat\mu_k)(x_{k,i}-Y_{k,i})^\top\theta_k \right|\\ 
   &\leq \sqrt{\frac{1}{2p}\sum_{i=1}^{2p} v^\top (Y_{k,i}-\widehat\mu_k)(Y_{k,i}-\widehat\mu_k)^\top v}\sqrt{\frac{1}{2p}\sum_{i=1}^{2p} \theta_k^\top (x_{k,i}-Y_{k,i})(x_{k,i} - Y_{k,i})^\top \theta_k}  \\ 
   &\leq \sqrt{\alpha} \|v\|_{\widehat\Sigma_k}\sqrt{\frac{p}{2p}\sum_{i=1}^{2p}\|x_{k,i}-Y_{k,i}\|_2^2}   \tag{$\|\theta_k\|_2^2\leq dH=p$} \\ 
   &= c_0\|v\|_{\widehat\Sigma_k}\varepsilon_k.   \tag{$\alpha = c_0^2/p$}
\end{align*}
Taking $v = \widehat\Sigma_k^{-1}m_k$, we get $\|m_k\|^2_{\widehat\Sigma_k^{-1}} \leq c_0 \|m_k\|_{\widehat\Sigma_k^{-1}}\varepsilon_k$ and hence 
\begin{align*}
    \|m_k\|_{\widehat\Sigma_k^{-1}} \leq c_0 \varepsilon_k. 
\end{align*}
Since $\Sigma_k\preceq2\widehat\Sigma_k$, we have  
\begin{align*}
 \left\|(\zeta^2\mathbf{I}_p+\alpha\widehat\Sigma_k)^{-1}
   m_k\right\|_{\Sigma_k}^2 &\leq 2\left\|(\zeta^2\mathbf{I}_p+\alpha\widehat\Sigma_k)^{-1}
   m_k\right\|_{\widehat\Sigma_k}^2\\
 &\le \frac{2}{\alpha^2} \|m_k\|_{\widehat\Sigma_k^{-1}}^2
 \leq \frac{2c_0^2\varepsilon_k^2}{\alpha^2}
 = \frac{2p^2\varepsilon_k^2}{c_0^2}.
\end{align*}
We also have 
\begin{align*}
 \left\|\zeta^2
   (\zeta^2\mathbf{I}_p+\alpha\widehat\Sigma_k)^{-1}
   \theta_k\right\|_{\Sigma_k}^2
 &\le2\zeta^4\theta_k^\top
   (\zeta^2\mathbf{I}_p+\alpha\widehat\Sigma_k)^{-1}
   \widehat\Sigma_k
   (\zeta^2\mathbf{I}_p+\alpha\widehat\Sigma_k)^{-1}\theta_k\\
   &\le \frac{2}{\alpha}\zeta^4\theta_k^\top
   (\zeta^2\mathbf{I}_p+\alpha\widehat\Sigma_k)^{-1/2}
   (\zeta^2\mathbf{I}_p+\alpha\widehat\Sigma_k)^{-1/2}\theta_k\\
 &\le\frac{2\zeta^2}{\alpha}\|\theta_k\|_2^2
 =\frac{2\zeta^2p^2}{c_0^2}.
\end{align*}
Combining these two bounds in \eqref{eq: bias deom} and using
$\varepsilon_k\ge2\zeta$ yields
\begin{equation*}
 \|\theta_k-\E_k[\widehat\theta_k]\|_{\Sigma_k}
 \le\frac{\sqrt2p\varepsilon_k}{c_0}
    +\frac{\sqrt{2}p\zeta}{c_0}
 \le\frac{5p\varepsilon_k}{2c_0}.
\end{equation*}

Moreover, $0\le L_k\le H$ and the leverage identity above imply
\begin{align*}
 \|\widehat\theta_k\|_{\Sigma_k}^2
 &\le2H^2(Y_{k,I_k}-\widehat\mu_k)^\top
   (\zeta^2\mathbf{I}_p+\alpha\widehat\Sigma_k)^{-1}
   \widehat\Sigma_k
   (\zeta^2\mathbf{I}_p+\alpha\widehat\Sigma_k)^{-1}
   (Y_{k,I_k}-\widehat\mu_k)\\
 &\le\frac{2H^2}{\alpha}
   (Y_{k,I_k}-\widehat\mu_k)^\top
   (\alpha\widehat\Sigma_k)^{-1}
   (Y_{k,I_k}-\widehat\mu_k)
 \le\frac{2H^2p^2}{c_0^2}.
\end{align*}
This proves the first two bounds in \eqref{eq:robust-bias-leverage-detail}. 

By \eqref{eq:partition-derivatives} and \eqref{eq:barrier-duality},
\begin{equation*}
 \nabla F_{\cO_k}(\mu_k)=-\eta Z_{k-1},
 \qquad
 \nabla^2F_{\cO_k}(\mu_k)=\Sigma_k^{-1},
\end{equation*}
and hence
$B_k=-10p\varepsilon_k\nabla F_{\cO_k}(\mu_k)/c_0$.
Equation~\eqref{eq:barrier-gradient-bound} gives
\begin{equation*}
 \|B_k\|_{\Sigma_k}
 \le\frac{10p^{3/2}\varepsilon_k}{c_0}.
\end{equation*}
The triangle inequality and $(a+b)^2\le2a^2+2b^2$, together with the
bound on $\|\widehat\theta_k\|_{\Sigma_k}$ above, prove the last bound in \eqref{eq:robust-bias-leverage-detail}. 
\end{proof}

\begin{lemma}[Bias--bonus compensation]
\label{lem:bias-bonus-compensation}
Under the assumptions of \cref{lem:corrected-estimator}, for every
$u\in\cO_k$, it holds that
\begin{equation}
  \ip{\mu_k-u}{\theta_k}
  \le\E_k[\inner{\mu_k-u,\widetilde\theta_k}]
       +\frac{60p^2\varepsilon_k}{c_0}.
  \label{eq:corrected-bias-cancellation}
\end{equation}
\end{lemma}

\begin{proof}[Proof of \cref{lem:bias-bonus-compensation}]
By \eqref{eq:partition-derivatives} and \eqref{eq:barrier-duality},
\begin{equation*}
 \nabla F_{\cO_k}(\mu_k)=-\eta Z_{k-1},
 \qquad
 \nabla^2F_{\cO_k}(\mu_k)=\Sigma_k^{-1},
\end{equation*}
and hence
$B_k=-10p\varepsilon_k\nabla F_{\cO_k}(\mu_k)/c_0$.
Local Cauchy--Schwarz, \cref{eq:robust-bias-leverage-detail}, and
\cref{eq:barrier-distance-comparison} give
\begin{align*}
 \langle\mu_k-u,\theta_k-\E_k[\widehat\theta_k]\rangle
 &\le\frac{5p\varepsilon_k}{2c_0}
   \|\mu_k-u\|_{\nabla^2F_{\cO_k}(\mu_k)}\\
 &\le\frac{10p\varepsilon_k}{c_0}
   \langle\mu_k-u,\nabla F_{\cO_k}(\mu_k)\rangle
   +\frac{60p^2\varepsilon_k}{c_0}\\
 &=-\langle\mu_k-u,B_k\rangle
   +\frac{60p^2\varepsilon_k}{c_0}.
\end{align*}
Since $B_k$ is $\cF_k$-measurable,
\begin{equation*}
 \langle\mu_k-u,\theta_k\rangle
 \le\E_k[\langle\mu_k-u,\widehat\theta_k-B_k\rangle]
   +\frac{60p^2\varepsilon_k}{c_0},
\end{equation*}
which proves \eqref{eq:corrected-bias-cancellation}.
\end{proof}

\subsection{Overall regret analysis}
\label{sec:outer-inputs}
\label{sec:good-event}

Let $\cF_k$ be the information available after preparation of episode
$k$ and before the policy draw, and write
$\E_k[\cdot]=\E[\cdot\mid\cF_k]$. Define $\chi_k$ to equal one if all
Bellman-confidence and moment-estimation events have held through this
preparation. Once it becomes zero, it remains zero. Thus $\chi_k$ is
$\cF_k$-measurable and, by \cref{thm:bellman-confidence}, the
moment-estimation guarantee, and a union bound over all calls,
\begin{equation}
 \Pp(\chi_k=1\text{ for every }k\in[K])\ge1-\delta.
 \label{eq:all-prefix-valid}
\end{equation}
where $\delta=2K^{-2}$. Whenever $\chi_k=1$,
$\cK_\zeta\subseteq\cO_k\subseteq R\Ball_2^p$ and the moment and
compiler guarantees hold. Write
$\xbar_k=\E_k[\phi^{\pi_k}]$ and set the auxiliary quantities to zero
when $\chi_k=0$.

\begin{lemma}[Cumulative estimated loss on nested regions]
\label{lem:nested-partition-bound}
Let $\cO_1\supseteq\cdots\supseteq\cO_K$ satisfy
$\cK_\zeta\subseteq\cO_k\subseteq R\Ball_2^p$. Let
$(\mu_k,\Sigma_k)$ be the exact mean and covariance of the density
proportional to $e^{-\eta Z_{k-1}^\top x}$ on $\cO_k$, where $Z_0=0$
and $Z_k=Z_{k-1}+\widetilde\theta_k$. If
$\eta\|\widetilde\theta_k\|_{\Sigma_k}\le1/2$ for every $k$, then for
each $u\in\cK$,
\[
 \sum_{k=1}^K\langle\mu_k-u,\widetilde\theta_k\rangle
 \le\frac{p\log(R/\zeta)}\eta
       +\eta\sum_{k=1}^K\|\widetilde\theta_k\|_{\Sigma_k}^2.
\]
\end{lemma}

\begin{proof}
Apply \cref{eq:selfconcordant-taylor} at $v=-\eta Z_{k-1}$ with
increment $-\eta\widetilde\theta_k$. Since
$r=\eta\|\widetilde\theta_k\|_{\Sigma_k}\le1/2$ and
$-r-\log(1-r)\le r^2$,
\[
 \Phi_{\cO_k}(-\eta Z_k)-\Phi_{\cO_k}(-\eta Z_{k-1})
 \le-\eta\langle\mu_k,\widetilde\theta_k\rangle
       +\eta^2\|\widetilde\theta_k\|_{\Sigma_k}^2.
\]
Since $\cO_{k+1}\subseteq\cO_k$, by
\cref{eq:entropic-barrier-def} we have
$\Phi_{\cO_{k+1}}\le\Phi_{\cO_k}$, and thus
\[
\begin{aligned}
 \Phi_{\cO_K}(-\eta Z_K)-\Phi_{\cO_1}(0)
 &=\sum_{k=1}^K
   [\Phi_{\cO_k}(-\eta Z_k)-\Phi_{\cO_k}(-\eta Z_{k-1})]\\
 &\quad+\sum_{k=1}^{K-1}
   [\Phi_{\cO_{k+1}}(-\eta Z_k)-\Phi_{\cO_k}(-\eta Z_k)]\\
 &\le-\eta\sum_{k=1}^K
   \langle\mu_k,\widetilde\theta_k\rangle
   +\eta^2\sum_{k=1}^K\|\widetilde\theta_k\|_{\Sigma_k}^2.
\end{aligned}
\]
Finally, $\Phi_{\cO_1}(0)\le\log\vol(R\Ball_2^p)$ and Jensen on
$u+\zeta\Ball_2^p\subseteq\cO_K$ give
\[
 \int_{\cO_K}e^{-\eta Z_K^\top x}\,dx
 \ge e^{-\eta Z_K^\top u}\vol(\zeta\Ball_2^p).
\]
Using $Z_K=\sum_k\widetilde\theta_k$ yields
\begin{equation}
 \sum_{k=1}^K\langle\mu_k-u,\widetilde\theta_k\rangle
 \le\frac{p\log(R/\zeta)}\eta
       +\eta\sum_{k=1}^K\|\widetilde\theta_k\|_{\Sigma_k}^2.
 \label{eq:robust-active-potential}
\end{equation}
\end{proof}

\begin{proof}[Proof of \cref{prop:paper-master}]
Define $\mathcal E_K
 :=\E\left[\sum_{k=1}^K\chi_k\varepsilon_k^2\right]$.
On the valid prefix, \cref{lem:bias-bonus-compensation} gives
\begin{equation}
 \langle\mu_k-u,\theta_k\rangle
 \le\E_k\left[\langle\mu_k-u,\widetilde\theta_k\rangle\right]
    +\frac{60p^2\varepsilon_k}{c_0}.
 \label{eq:main-bias-compensation}
\end{equation}
Furthermore,
\begin{equation}
 \xbar_k-\mu_k=(\widehat\mu_k-\mu_k)
 +\frac1{2p}\sum_i(x_{k,i}-Y_{k,i}).
 \label{eq:mean-displacement-decomposition}
\end{equation}
Since $p=dH$, the RMS approximation and moment bounds imply
\begin{equation}
 \langle\xbar_k-u,\theta_k\rangle
 \le\E_k\left[\langle\mu_k-u,\widetilde\theta_k\rangle\right]+e_k,
 \qquad
 e_k:=\left(\frac{60p^2}{c_0}+\sqrt p\right)\varepsilon_k
      +2R\sqrt p\,\delta_\mu.
 \label{eq:main-one-round-comparison}
\end{equation}

Write
$r_k=\langle\phi^{\pi_k}-u,\theta_k\rangle$. Since $\chi_k$ is
$\cF_k$-measurable,
\begin{equation}
 \E\left[\sum_{k=1}^K\chi_kr_k\right]
 =\E\left[\sum_{k=1}^K\chi_k
   \langle\xbar_k-u,\theta_k\rangle\right].
 \label{eq:masked-regret-mean}
\end{equation}
Multiplying \cref{eq:main-one-round-comparison} by $\chi_k$ and taking
expectations gives
\begin{equation}
 \E\left[\sum_{k=1}^K\chi_kr_k\right]
 \le
 \E\left[\sum_{k=1}^K\chi_k
   \langle\mu_k-u,\widetilde\theta_k\rangle\right]
 +\E\left[\sum_{k=1}^K\chi_ke_k\right].
 \label{eq:masked-regret-comparison}
\end{equation}
Here we used the tower property, since
$\chi_k(\mu_k-u)$ is $\cF_k$-measurable.

By \cref{lem:corrected-estimator},
\[
 \norm{\widetilde\theta_k}_{\Sigma_k}^2
 \le\frac{4H^2p^2}{c_0^2}
    +\frac{200p^3\varepsilon_k^2}{c_0^2}.
\]
On the valid prefix, $\varepsilon_k\le6\sqrt H$. Provided
$\eta\le c_0/(128p^{3/2}\sqrt H)$,
\cref{eq:robust-bias-leverage-detail} gives
\[
 \eta\norm{\widetilde\theta_k}_{\Sigma_k}
 \le\eta\!\left(\frac{\sqrt2Hp}{c_0}
                 +\frac{10p^{3/2}\varepsilon_k}{c_0}\right)
 \le\frac{60+\sqrt{2/d}}{128}
 \le\frac{60+\sqrt2}{128}<\frac12,
\]
which verifies the condition of the nested-region potential bound.
On every sample path, the indices with $\chi_k=1$ form a prefix.
Applying \cref{lem:nested-partition-bound} to this realized prefix and
then taking expectations yields
\[
 \E\left[\sum_{k=1}^K\chi_k
   \langle\mu_k-u,\widetilde\theta_k\rangle\right]
 \le\frac{p\log(R/\zeta)}\eta
   +\frac{4\eta H^2p^2K}{c_0^2}
   +\frac{200\eta p^3}{c_0^2}\mathcal E_K.
\]
Furthermore,
\[
 \E\left[\sum_k\chi_k\varepsilon_k\right]
 \le\sqrt{K\mathcal E_K}.
\]
Substitution into \cref{eq:masked-regret-comparison}, followed by
$p=dH$, gives
\[
 \E\left[\sum_{k=1}^K\chi_kr_k\right]
 \le\widetilde O\!\left(
 \frac{dH}{\eta}+\eta d^2H^4K
 +\eta d^3H^3\mathcal E_K
 +d^2H^2\sqrt{K\mathcal E_K}
 +R\sqrt p\,\delta_\mu K
 \right).
\]

Finally, $r_k\le H$, and by \cref{eq:all-prefix-valid},
\[
\begin{aligned}
 \E\left[\sum_{k=1}^Kr_k\right]
 &=\E\left[\sum_{k=1}^K\chi_kr_k\right]
   +\E\left[\sum_{k=1}^K(1-\chi_k)r_k\right]\\
 &\le\E\left[\sum_{k=1}^K\chi_kr_k\right]+HK\delta.
\end{aligned}
\]
Using $\delta_\mu=2c_0/\sqrt K$, $\delta=2/(K+1)^2$,
$R=\sqrt H+K^{-1/2}$, and $p=dH$, we have
\[
 R\sqrt p\,\delta_\mu K+HK\delta
 =O\!\left(R\sqrt{pK}+\frac HK\right)
 =O\!\left(H\sqrt{dK}\right).
\]
The unconstrained minimizer of the two $\mathcal E_K$-independent
$\eta$-terms is $(dH^3K)^{-1/2}$. To also ensure the local-norm
condition above, choose
\[
 \eta=\frac{c_0}{128}
 \min\!\left\{(dH^3K)^{-1/2},d^{-3/2}H^{-2}\right\}.
\]
The second branch is exactly the required stability cap. Moreover, the
pathwise bound $\mathcal E_K\le36HK$ gives
\[
 \eta d^3H^3\mathcal E_K
 \le O\!\left(d^{3/2}H\mathcal E_K\right)
 \le O\!\left(d^2H^2\sqrt{K\mathcal E_K}\right).
\]
Thus the $\eta\mathcal E_K$ term is absorbed by the
$d^2H^2\sqrt{K\mathcal E_K}$ term. In either branch,
$H\sqrt{dK}=O(\eta d^2H^4K)$. If $K\ge d^2H$, the first branch applies
and
\[
 \frac{dH}{\eta}+\eta d^2H^4K
 =O\!\left(d^{3/2}H^{5/2}\sqrt K\right).
\]
If $K<d^2H$, the trivial bound $\Reg_K\le HK$ is no larger than the
same expression. Therefore
\[
 \Reg_K
 \le\widetilde O\!\left(d^{3/2}H^{5/2}\sqrt K
 +d^2H^2\sqrt{K\mathcal E_K}\right).
\]
This proves \cref{prop:paper-master}.
\end{proof}

%% file: appendices/05_support.tex
\newpage
\section{Omitted Proofs in \cref{sec:paper-compiler}}
\label{sec:support-oracle}
\label{sec:transition-confidence}
\label{app:confidence}

\begin{lemma}[Averaging recursion]
\label{lem:averaging-recursion}
Let $a_0,a_1,\ldots,a_M$ and $b_1,\ldots,b_M$ be real sequences,
and let $c\ge0$. Suppose that, for every $j\in[M]$,
\[
 a_j
 \le\left(1-\frac1j\right)a_{j-1}
    +\frac{b_j}{j}+\frac{c}{j^2}.
\]
Then, for every $m\in[M]$,
\begin{equation}
 a_m\le\frac1m\sum_{j=1}^m b_j
          +\frac{c(1+\log m)}m.
 \label{eq:averaging-recursion-bound}
\end{equation}
\end{lemma}

\begin{proof}
Multiplying the recursion by $j$ gives
\[
 ja_j\le(j-1)a_{j-1}+b_j+\frac{c}{j}.
\]
Summing from $j=1$ to $m$ telescopes the terms involving $a_j$ and
yields
\[
ma_m\le\sum_{j=1}^m b_j+c\sum_{j=1}^m\frac1j,
\]
which proves \cref{eq:averaging-recursion-bound} because
$\sum_{j=1}^m j^{-1}\le1+\log m$.
\end{proof}

\begin{theorem}[Uniform Bellman confidence (Lemma~D.4 of \cite{jin2020})]
\label{thm:bellman-confidence}
Fix $\delta\in(0,1)$ and $\epsilon\in(0,1)$. For each $h\in[H]$, let
$C\ge1$ and let $\mathcal F$ be a deterministic class of
scalar functions from $\cS_{h+1}$ to $[-C,C]$. Define the
relative covering number
\[
 N_{\mathcal F}(\epsilon)
 :=
 \mathcal N_\infty(C\epsilon,\mathcal F),
\]
where $\mathcal N_\infty$ denotes the sup-norm covering number. Assume
$N_{\mathcal F}(\epsilon)<\infty$. There is a universal constant $c_1>0$
such that, if
\begin{equation}
 \beta
 \ge c_1\left(
 \sqrt{d\log K
       +\log\!\left(\frac{HN_{\mathcal F}(\epsilon)}{\delta}\right)}
 +K\epsilon+\sqrt d\right),
 \label{eq:beta-def}
\end{equation}
then, with probability at least $1-\delta$, simultaneously for every
$k\in[K]$, $h\in[H]$, $f\in\mathcal F_{h+1}$, and
$(s,a)\in\cS_h\times\cA$,
\begin{equation}
 |(\widehat P_{k,h}-P_h)f(s,a)|
 \le  \beta C
 \norm{\phi_h(s,a)}_{\Lambda_{k,h}^{-1}}.
 \label{eq:uniform-bellman-confidence}
\end{equation}
\end{theorem}

\newcommand{\const}{\mathsf{const}}
\begin{proof}[Proof of \cref{thm:certified-lmo}]

All three tasks in Algorithm~\ref{alg:lmo} are based on least-squares
value iteration. The key is to show \eqref{eq:uniform-bellman-confidence}
for all three tasks, each corresponding to a different function class
$\cF$. Fix a layer $h$, write
$\alpha=(\alpha_1,\alpha_2,\alpha_3)$, and define
\begin{align*}
 \nu_h(s,a;\Lambda)
 &=\sqrt{\phi_h(s,a)^\top\Lambda^{-1}\phi_h(s,a)} \\ 
 q_h(s,a;\Lambda,\vartheta,\alpha)
 &=\CLIP_{[\alpha_1,\alpha_2]}\!\left(
 \phi_h(s,a)^\top\vartheta
 +\alpha_3\nu_h(s,a;\Lambda)\right),\\  
 \pi_h(a\mid s;\Lambda,\widetilde\vartheta)
 &\propto\exp\!\left(
 -q_h(s,a;\Lambda,\widetilde\vartheta,\alpha^{\pl})/\tau\right),\\
 v_h(s;\Lambda,\vartheta,\widetilde\vartheta,\alpha)
 &=\sum_a\pi_h(a\mid s;\Lambda,\widetilde\vartheta)
 q_h(s,a;\Lambda,\vartheta,\alpha),
\end{align*}
where
\[
 \alpha^{\pl}=(-\sqrt H,\sqrt H,-(\sqrt H+1)\beta),\qquad
 \alpha^{\feat}=(-1,1,0),\qquad
 \alpha^{\rm unc}=(0,H-h+1,1+H\beta).
\]
The function classes for the three tasks are respectively
\begin{align*}
 &s\mapsto-\tau
 \log\sum_a \exp\!\left(
 -q_h(s,a;\Lambda,\vartheta,\alpha^{\pl})/\tau\right)
 \tag{planning} \\ 
 &s\mapsto v_h(s;\Lambda,\vartheta,\widetilde\vartheta,\alpha^{\feat})
 \tag{feature evaluation}\\
 &s\mapsto v_h(s;\Lambda,\vartheta,\widetilde\vartheta,\alpha^{\rm unc})
 \tag{uncertainty evaluation}.
\end{align*}
Denote these classes by $\mathcal F_h^{\pl}$,
$\mathcal F_h^{\feat}$, and $\mathcal F_h^{\rm unc}$, respectively,
and set their terminal classes to $\{0\}$, with terminal scale one.

\paragraph{Bounding the covering number of each classes. }
\mbox{Clipping and $\tau\log A\le1$ give the uniform bounds}
\begin{align}
 f\in\mathcal F_h^{\pl}&:~~\norm{f}_\infty\le\sqrt H+1, \nonumber \\ 
 f\in\mathcal F_h^{\feat}&:~~\norm{f}_\infty\le1, \nonumber  \\ 
 f\in\mathcal F_h^{\rm unc}&:~~\norm{f}_\infty\le H-h+1. \label{eq:function-scales}
\end{align}
Write
$C_h^{\pl}=\sqrt H+1$, $C_h^{\feat}=1$, and
$C_h^{\rm unc}=H-h+1$ for these scales.
Therefore,  
\[
 \left\|\Lambda_{k,h}^{-1}
 \sum_{(u,s')\in\cD_{k,h}}u f(s')\right\|_2
 \le\sum_{(u,s')\in\cD_{k,h}}\norm u_2|f(s')|
 \le3KH.
\]
Including the immediate reward/cost gives
$\|\vartheta\|_2,\|\widetilde\vartheta\|_2\le3KH+1$. Also, $\mathbf{I}_d\preceq \Lambda_{k,h}\preceq (K+1)\mathbf{I}_d$. 
Thus, each class has at most
\[
 n=\frac{d(d+1)}2+2d\le3d^2
\]
scalar parameters, and each of them is bounded in absolute value by
$\overline{R}=\max\{K+1,3KH+1\}$.
Moreover,
\[
 |\nu_h(s,a;\Lambda)-\nu_h(s,a;\Lambda')|
 \le\frac12\norm{\Lambda-\Lambda'}_{\rm F},
\]
so clipping gives, for each fixed $\alpha$,
\[
 \norm{q_h(\cdot;\Lambda,\vartheta,\alpha)
       -q_h(\cdot;\Lambda',\vartheta',\alpha)}_\infty
 \le \norm{\vartheta-\vartheta'}_2
    +\frac{|\alpha_3|}{2}\norm{\Lambda-\Lambda'}_{\rm F}.
\]
The softmax map satisfies
\[
 \sup_s\norm{\pi_h(\cdot\mid s;q)-\pi_h(\cdot\mid s;q')}_1
 \le\frac2\tau\norm{q-q'}_\infty.
\]
Hence, if $C$ denotes the corresponding scale in
\eqref{eq:function-scales},
\[
 \norm{v_h-v_h'}_\infty
 \le \norm{q_h-q_h'}_\infty
   + C\sup_s\norm{\pi_h(\cdot\mid s)-\pi_h'(\cdot\mid s)}_1.
\]
The planning soft minimum is nonexpansive in sup norm. Since
$C\ge1$, $|\alpha_3|/C\le1+H\beta$, and
$(\sqrt H+1)\beta\le2H\beta$, the evaluation is $L$-Lipschitz with respective to its parameters, with 
\[
 L=\frac{8(1+H\beta)}{\tau}.
\]
A parameter grid of mesh $\epsilon/(nL)$ gives a
$C\epsilon$-cover. Therefore,
\begin{equation}
 \log\mathcal N_\infty
 \bigl(C\epsilon,\mathcal F\bigr)
 =O\!\left(
 n\log\!\left(1+\frac{\overline RnL}{\epsilon}\right)
 \right).
 \label{eq:continuation-covering-number}
\end{equation}

\paragraph{Bounding error.}
Take
\[
 \delta=K^{-2},\qquad
 \epsilon_{\rm net}=K^{-2},\qquad
 \beta=C_\beta d\sqrt{\log(KdHA)}.
\]
Since $\tau^{-1}=H\sqrt K\log(eA)$,
\eqref{eq:continuation-covering-number} gives
\[
 \log N_{\mathcal F^t}(\epsilon_{\rm net})
 =O\!\left(d^2\log(KdHA)\right).
\]
Thus a sufficiently large universal $C_\beta$ makes
\eqref{eq:beta-def} hold, and $\beta=\widetilde O(d)$. Applying
\cref{thm:bellman-confidence} directly to the three classes, with
failure probability $\delta/3$ each, and taking a union bound
yields
\begin{align}
 |(\widehat P_{k,h}-P_h)V_{h+1}^{\pl}(s,a)|
 &\le(\sqrt H+1)\beta\nu_{k,h}(s,a),  \label{eq: good even pl}\\
 |(\widehat P_{k,h}-P_h)V_{h+1}^{\feat,m}(s,a)|
 &\le\beta\nu_{k,h}(s,a),\qquad m\in[p],   \label{eq: good even eval}\\
 |(\widehat P_{k,h}-P_h)\widehat U_{h+1}(s,a)| 
 &\le H\beta\nu_{k,h}(s,a).   \label{eq: good even}
\end{align}

Fix one call and let $\pi$ be its returned policy. Write
\[
 U_k^\pi
 =\E_{P,\pi}\!\left[
   \sum_{h=1}^H\nu_{k,h}(S_h,A_h)
 \right].
\]

By backward induction \citep{jin2020} on the update in Line~\ref{line:support-planning} with confidence bound \eqref{eq: good even pl}, and that $V_h^{\pl}(s)\le\min_aQ_h^{\pl}(s,a)$, we have $V_h^{\pl}\le \min_{\pi}V_h^{\pi}$. Therefore, 
\begin{equation}
 V_1^{\pl}(s_1)\le\min_{u\in\cK}\bar c^\top u.
 \label{eq:planning-lower-bound}
\end{equation}
Likewise, backward induction on the update in Line~\ref{line:support-uncertainty-backup} with confidence bound \eqref{eq: good even} gives
\begin{equation}
 U_k^\pi\le\widehat U_1(s_1).
 \label{eq:uncertainty-lower-bound}
\end{equation}

Similarly, backward induction on the update in Line~\ref{line:support-feature-average}
with coordinatewise confidence bounds \eqref{eq: good even eval} gives
\begin{equation}
 \|\widehat\phi-\phi^\pi\|_2
 \le\sqrt p\,\beta U_k^\pi \leq \sqrt{p}\beta \widehat U_1(s_1).
 \label{eq:feature-evaluation-bound}
\end{equation}
The feature-estimation error $\|\widehat\phi-\phi^\pi\|_2$ is at most $2\sqrt H$ because
$\widehat\phi,\phi^\pi\in\sqrt H\Ball_2^p$.
Thus, 
\begin{align*}
      \|\widehat\phi-\phi^\pi\|_2
 \le \min\left\{
5\sqrt p\,\beta\widehat U_1(s_1)+H\tau\log A,\;
2\sqrt H
\right\} = \Delta
\end{align*}
proving the first bound in \eqref{eq:two-bound}.

Comparing the planning recursion with the true value of the returned
Gibbs policy, each layer contributes at most
$2(\sqrt H+1)\beta\nu_{k,h}$ from transition error and at most
$\tau\log A$ from the soft minimum. Hence
\begin{equation}
 \bar c^\top\phi^\pi-V_1^{\pl}(s_1)
 \le2(\sqrt H+1)\beta U_k^\pi+H\tau\log A.
 \label{eq:planning-gap-bound}
\end{equation}

The same uncertainty recursion also gives
\begin{equation}
 \widehat U_1(s_1)\le2\beta H U_k^\pi.
 \label{eq:uncertainty-upper-bound}
\end{equation}

Combining
\eqref{eq:feature-evaluation-bound}, \eqref{eq:planning-lower-bound}, \eqref{eq:planning-gap-bound}
and using $p=dH$ yields
\begin{align}
 \bar c^\top\widehat\phi-\min_{u\in\cK}\bar c^\top u \le 5\sqrt p\,\beta U_k^\pi+H\tau\log A \leq 5\sqrt p\,\beta \widehat U_1(s_1)+H\tau\log A.
\label{eq:unclipped-linear-optimization}
\end{align}
The normalized optimization gap
$\bar c^\top\widehat\phi-\min_{u\in\cK}\bar c^\top u$ is at most $2\sqrt H$ because
$\widehat\phi \in\sqrt H\Ball_2^p$,
$\cK\subseteq\sqrt H\Ball_2^p$, and $\|\bar c\|_2\le1$.
Thus,
\begin{align*}
    \bar c^\top\widehat\phi-\min_{u\in\cK}\bar c^\top u \le \min\left\{
5\sqrt p\,\beta\widehat U_1(s_1)+H\tau\log A,\;
2\sqrt H
\right\} = \Delta
\end{align*}
proving the second bound in \eqref{eq:two-bound}.

Finally,
\eqref{eq:uncertainty-upper-bound} gives
\[
 \Delta
 \le10\sqrt p\,\beta^2H U_k^\pi+H\tau\log A,
\]
proving the final bound.
\end{proof}

%% file: appendices/07_computation.tex
\newpage
\section{Omitted Proofs in \cref{sec:paper-computation}}
\label{sec:computation}
\label{app:computation-details}

\subsection{Moment Estimates}
\label{app:moment-details}

We let \Alg{MomentEstimate} to be a procedure described in \cite{ito2019oracle}, which takes a convex function $f: \R^p\to \R$ and a convex body $\mathcal{P}\subset \R^p$ as input, and output estimators of the mean and covariance of the distribution $q(x)\propto \exp(-f(x))\mathbf{1}\{x\in\mathcal{P}\}$ within polynomial time. In our case $f(x) = \eta x^\top Z_{t-1}$ and $\mathcal{P}=\cO$. Its guarantee is the following: 
\begin{lemma}[Corollary~1 of \cite{ito2019oracle}] Let $\mathcal{P}\subset \R^p$ be a convex body with non-zero Lubesgue measure, and let $q(x)\propto \exp(-f(x))\ind\{x\in \mathcal{P}\}$ be a log-concave distribution with some convex function $f: \mathcal{P}\to \R$. Let $\mu$ and $\Sigma$ denote the mean and covariance of $q$. Then there is an algorithm that outputs a vector $\widehat\mu\in\R^p$ and a symmetric matrix $\widehat\Sigma\in\R^{p\times p}$ satisfying $\frac{1}{2}\Sigma \preceq \widehat\Sigma \preceq 2\Sigma$ and $\|\widehat\mu - \mu\|_{\mathrm{Cov}(q)^{-1}}\leq \varepsilon$ with probability at least $1-\delta$. The computational time, the number of calls for the membershpi oracle of $\mathcal{P}$, and the value oracle of $f$ are bounded by $\poly(p,\frac{1}{\varepsilon}, \log\frac{1}{\delta})$.  
\end{lemma}
We let $\varepsilon=\delta_\mu=2c_0/\sqrt K$, which implies the estimation accuracy 
\begin{align*}
    \|\widehat\mu_k - \mu_k\|_{\Sigma_k^{-1}} \leq \delta_\mu, \qquad \frac{1}{2}\Sigma_k \preceq \widehat\Sigma_k \preceq 2\Sigma_k
\end{align*}
after calling \Alg{MomentEstimate} in Line~\ref{line: moment estimate} of \cref{alg:paper-core}. 

\subsection{Number of cuts}

\begin{lemma}[Containment of the estimated covariance ellipsoid]
\label{lem:estimated-ellipsoid-containment}
Let $q$ be a full-dimensional log-concave distribution supported on a
convex body $\cO\subseteq\R^p$, with mean $\mu$ and covariance $\Sigma$.
Suppose $\widehat\mu$ and $\widehat\Sigma$ satisfy
$
 \frac12\Sigma\preceq\widehat\Sigma\preceq2\Sigma$ and $
 \norm{\widehat\mu-\mu}_{\Sigma^{-1}}\le\frac1{12e}.
$
Then
\begin{equation}
 \widehat\mu+\frac1{16e}\widehat\Sigma^{1/2}\Ball_2^p
 \subseteq
 \mu+\frac1{2e}\Sigma^{1/2}\Ball_2^p
 \subseteq\cO.
 \label{eq:estimated-ellipsoid-containment}
\end{equation}
\end{lemma}

\begin{proof}
We first show that
\begin{equation}
 \mu+\frac1e\Sigma^{1/2}\Ball_2^p\subseteq\cO.
 \label{eq:exact-covariance-ellipsoid-containment}
\end{equation}
Otherwise, separation gives a nonzero vector $a$ and a point $y$ with
$\norm{y-\mu}_{\Sigma^{-1}}\le1/e$ such that
$a^\top y<\inf_{x\in\cO}a^\top x$. For $X\sim q$, denote the random scalar 
\[
 S=\frac{a^\top(X-\mu)}{\sqrt{a^\top\Sigma a}} 
\]
and denote the density function of $S$ as $q_S$. 
Then $S$ is centered, variance-one, and log-concave. Moreover, its
support is contained in $[u_{\min},\infty)$: Indeed,
strict separation and compactness of $\cO$ give
\[
 u_{\min}:=\frac{\inf_{x\in\cO}a^\top(x-\mu)}
                     {\sqrt{a^\top\Sigma a}}
 >\frac{a^\top(y-\mu)}{\sqrt{a^\top\Sigma a}}
 \ge-\frac1e,
\]
where the last inequality follows from Cauchy--Schwarz in the
$\Sigma$-norm.
Lemmas~5.4 and~5.5(a) of \citet{lovasz2007} give
$\Pp(S<0)\ge1/e$ and $q_S(s)\leq 1$ for any $s$. However, the support condition would instead imply
\[
 \Pp(S<0)
 =\int_{u_{\min}}^0 q_S(t)\,dt
 \le |u_{\min}|
 <\frac1e,
\]
a contradiction. This proves
\cref{eq:exact-covariance-ellipsoid-containment}.

Now fix $v\in\Ball_2^p$. The covariance assumption gives
$\|\Sigma^{-1/2}\widehat\Sigma^{1/2}\|_{\op}\le\sqrt2$.
Therefore
\[
\begin{aligned}
 \norm{\widehat\mu+\frac1{16e}\widehat\Sigma^{1/2}v-\mu}_{\Sigma^{-1}}
 &\le\norm{\widehat\mu-\mu}_{\Sigma^{-1}}
   +\frac1{16e}
     \norm{\Sigma^{-1/2}\widehat\Sigma^{1/2}v}_2\\
 &\le\frac1{12e}+\frac{\sqrt2}{16e}
 <\frac1{2e}.
\end{aligned}
\]
The first inclusion in \cref{eq:estimated-ellipsoid-containment}
follows, and the second follows from
\cref{eq:exact-covariance-ellipsoid-containment}.
\end{proof}

\begin{lemma}[Number of cuts]
\label{lem:local-cut-progress}
Suppose every invoked moment call, with exact mean $\mu$ and covariance
$\Sigma$, returns estimates satisfying
$\frac12\Sigma\preceq\widehat\Sigma\preceq2\Sigma$ and
$\norm{\widehat\mu-\mu}_{\Sigma^{-1}}\le1/(12e)$.
Then the number of cuts is upper bounded by
$\widetilde O(pK+\sqrt H K^{5/2})$.
\end{lemma}

\begin{proof}
By the definition of $\widehat\theta_k$, we have $\|\widehat\theta_k\|_2 \leq \frac{2RH}{\zeta^2}$. Since the update is $Z_k = (1-\frac{10\eta p\varepsilon_k}{c_0})Z_{k-1} + \widehat\theta_k$, we have $\|Z_k\|\leq \|Z_{k-1}\| + \|\widehat \theta_k\|$. By induction, it holds that $\|Z_k\|\leq \frac{2RHK}{\zeta^2}$ for all $k$.

By \cref{lem:estimated-ellipsoid-containment}, every target used during
preparation belongs to the current region $\cO$ and satisfies
$\norm{Y_i-\mu}_{\Sigma^{-1}}<1/(2e)$.
The local-cut argument of
\citet[Lemma~5]{ito2019oracle} shows that each cut removes at least a
$\frac{1}{2e}$ fraction of the current mass. Let $N_k$ be the number of new cuts added in episode $k$, and let $\cO_i, \cO_f$ be the set $\cO$ at the start and the end of episode $k$. Then we have 
\begin{align*}
 \vol(\zeta\Ball_2^p)e^{-\eta R\|Z_{k-1}\|}
 &\le \int_{\cO_f} e^{-\eta x^\top Z_{k-1}} dx  \\ 
 &\le \left(1-\frac{1}{2e}\right)^{N_k} \int_{\cO_i} e^{-\eta x^\top Z_{k-1}} dx  \\ 
 &\leq \left(1-\frac{1}{2e}\right)^{N_k}
       \vol(R\Ball_2^p)e^{\eta R\|Z_{k-1}\|}. 
\end{align*}
Taking logarithms gives
\[
 N_k 
 \le O\left(p\log(R/\zeta)+\eta R\|Z_{k-1}\|\right)
 \le O\left(p\log(R/\zeta)+\frac{\eta R^2 KH}{\zeta^2}\right)
 \leq \widetilde{O}\big(p+\sqrt H K^{3/2}\big).
\]
Summing over episodes gives the desired bound. 
\end{proof}

\subsection{Implementation of \cref{alg:lmo}}
\label{app:complexity}

Although \cref{alg:lmo} defines value functions on the entire
state space, it can be executed without enumerating states or
state--action pairs. Below is an efficient way to implement \cref{alg:lmo}. We just use the computation of $V_h^{\pl}$ as example; the other two tasks follow the same manner. 

For any function $f:\cS_{h+1}\to\R$, define: 
\[
 \widehat w_h(f):=\Lambda_{k,h}^{-1}\sum_j u_j f(s'_j)^\top
 \in\R^{d\times m}.
\]
Then, for every queried pair $(s,a)$,
\[
 \widehat P_{k,h}f(s,a)=\phi_h(s,a)^\top \widehat w_h(f).
\]
Below, we show that we only need to keep track of the vector 
$\widehat w_h(V_{h+1}^{\pl})$, or equivalently $\vartheta_h = \bar{c}_h + \widehat w_h(V_{h+1}^{\pl})$ during the execution of \cref{alg:lmo}. 
They are computed by backward induction.

Suppose $\vartheta_h$ has already been computed. For every
state $s'_j$ appearing in $\cD_{k,h-1}$, query
$\phi_h(s'_j,a)$ for each $a\in\cA$ and compute
\[
\begin{aligned}
 Q_h^{\pl}(s'_j,a)
 &=\CLIP_{[-\sqrt H,\sqrt H]}\!\left(
   \phi_h(s'_j,a)^\top\vartheta_h
   -(\sqrt H+1)\beta
   \norm{\phi_h(s'_j,a)}_{\Lambda_{k,h}^{-1}}
   \right),\\
 V_h^{\pl}(s'_j)
 &=-\tau\log\sum_{a\in\cA}
   \exp\!\left(-Q_h^{\pl}(s'_j,a)/\tau\right).
\end{aligned}
\]
These are the values needed to form
$\widehat w_{h-1}(V_h^{\pl})$ and hence
\[
 \vartheta_{h-1}=\bar c_{h-1}+\widehat w_{h-1}(V_h^{\pl}).
\]
Thus, $\vartheta_h$ can be computed in a backward induction manner. This then allows us to evaluate $Q_h^\pl(s,a)$ and $V_h^\pl(s)$ efficiently for any queried $(s,a)$.